\documentclass[11pt,reqno]{article}
\usepackage{graphicx} 
\usepackage[english]{babel}
\usepackage[margin=1in]{geometry}
\usepackage{amsmath}
\usepackage{amssymb}
\usepackage{graphicx}
\usepackage[colorlinks=true, allcolors=blue]{hyperref}
\usepackage[utf8]{inputenc}
\usepackage{verbatim}
\usepackage{caption,subcaption}
\usepackage[overload]{empheq}
\usepackage{cleveref} 
\usepackage{amsthm}
\usepackage{epsf, epsfig, hyperref, mathtools,amsthm,bbm,xcolor,pifont}
\usepackage{algorithm}
\usepackage{algpseudocode}
\usepackage{bm}

\numberwithin{equation}{section}

\newtheorem{remark}{Remark}[section]
\newtheorem{lemma}{Lemma}[section]
\newtheorem{theorem}{Theorem}[section]
\newtheorem{definition}{Definition}[section]
\newtheorem{proposition}{Proposition}[section]
\newtheorem{corollary}{Corollary}[section]

\def \E{\mathbb{E}}

\def \N{\mathbb{N}}
\def \P{\mathbb{P}}

\def \R{\mathbb{R}}
\def \Pc{\mathcal{P}}
\def \Xc{\mathcal{X}}
\def \Yc{\mathcal{Y}}
\def \eps{\varepsilon}
\newcommand{\mud}{\mu_{\operatorname{d}}}
\newcommand{\Div}{{\operatorname{div}}}

\title{Generative Modeling by Minimizing the Wasserstein-2 Loss}
\author{Yu-Jui Huang\thanks{
Department of Applied Mathematics, University of Colorado, Boulder, CO 80309-0526, USA, email: \texttt{yujui.huang@colorado.edu}. Partially supported by National Science Foundation (DMS-2109002).
}
\and
Zachariah Malik\thanks{
Department of Applied Mathematics, University of Colorado, Boulder, CO 80309-0526, USA, email: \texttt{zachariah.malik@colorado.edu}.
}
}

\begin{document}
\maketitle

\begin{abstract}
This paper develops a generative model by minimizing the second-order Wasserstein loss (the $W_2$ loss) through a {\it distribution-dependent} ordinary differential equation (ODE), whose dynamics involves the Kantorovich potential associated with the true data distribution and a current estimate of it. A main result shows that the time-marginal laws of the ODE form a gradient flow for the $W_2$ loss, which converges exponentially to the true data distribution. 
An Euler scheme for the ODE is proposed and it is shown to recover the gradient flow for the $W_2$ loss in the limit. An algorithm is designed by following the scheme and applying {\it persistent training}, which naturally fits our gradient-flow approach. In both low- and high-dimensional experiments, our algorithm 
outperforms Wasserstein generative adversarial networks by increasing the level of persistent training appropriately. 
\end{abstract}

\smallskip


\section{Introduction}
From observed data points, a generative model aims at {learning} the unknown data distribution $\mud$, a probability measure on $\R^d$, so as to {generate} new data points indistinguishable from the observed ones. The success of it demands a rigorous underlying theory.



In this paper, starting with any initial guess $\mu_0$ of $\mud$, we measure the discrepancy between $\mu_0$ and $\mud$ by the second-order Wasserstein distance (i.e., $W_2(\cdot,\cdot)$ in \eqref{Eq: Kantorovich Relaxation} below) and strive to minimize the discrepancy efficiently. As the map $\mu\mapsto W^2_2(\mu,\mud)$ is convex, it is natural to ask if this minimization (or, learning of $\mud$) can be done by {\it gradient descent}, the traditional wisdom for convex minimization. The crucial question is how we should define ``gradient'' in the present setting. 

Thanks to subdifferential calculus for probability measures, widely studied under the $W_2$ distance in Ambrosio et al.\ \cite{Ambrosio08}, 
the Fr\'{e}chet subdifferential of $\mu\mapsto J(\mu):= \frac12 W^2_2(\mu,\mud)$, 
denoted by $\partial J(\mu)$, always exists; see Definition~\ref{def:subdifferential} and the discussion below it. Hence, we can take the ``gradient of  $J$'' evaluated at $\mu$ to be an element of $\partial J(\mu)$ and the resulting gradient flow 
converges exponentially to $\mud$ under $W_2$ (Proposition~\ref{prop:GF exists}). What is {\it not} answered by the full-fledged $W_2$-theory is how this gradient flow can actually be computed, for $\mud$ to be found realistically; see Remark~\ref{rem:how to simulate?}. 

To materialize the gradient flow, we plan to simulate a gradient-descent ordinary differential equation (ODE) with distribution dependence. Let $Y$ be an $\R^d$-valued process satisfying
\begin{eqnarray}\label{Y inclusion}
\frac{dY_t}{dt} \in -\partial J(\mu^{Y_t}), \quad \mu^{Y_0}=\mu_0,
\end{eqnarray}
where $\mu^{Y_t}$ denotes the law of $Y_t$ at $t\ge 0$. This stipulates that $Y_t$ evolves instantaneously along a ``negative gradient'' of $J$, evaluated at the current distribution $\mu^{Y_t}$. 
Note that \eqref{Y inclusion} is an inclusion, but not an equation, as $\partial J(\mu^{Y_t})$ in general may contain multiple elements. 
Notably, for the special case where $\mu^{Y_t}$ is absolutely continuous with respect to (w.r.t.) the Lebesgue measure (written $\mu^{Y_t}\ll \mathcal L^d$), $\partial J(\mu^{Y_t})$ reduces to a singleton containing only $\nabla \phi_{\mu^{Y_t}}^{\mud}$ (Lemma~\ref{lem:subdifferential J}), where $\phi_{\mu^{Y_t}}^{\mud}:\R^d\to\R^d$ is a {\it Kantorovich potential} from $\mu^{Y_t}$ to $\mud$ (which, by definition, maximizes the dual formulation of $W_2(\mu^{Y_t},\mud)$). 
Hence, if ``$\mu^{Y_t}\ll \mathcal L^d$ for all $t\ge 0$'' holds, \eqref{Y inclusion} will take the concrete form
\begin{eqnarray}\label{Y}
dY_t = -\nabla \phi_{\mu^{Y_t}}^{\mud}(Y_t) dt,\quad \mu^{Y_0} = \mu_0. 
\end{eqnarray}
This is a {\it distribution-dependent} ODE. At time 0, the law of $Y_0$ is given by $\mu_0$, an initial distribution satisfying $\mu_0\ll \mathcal L^d$. This initial randomness trickles through the ODE dynamics in \eqref{Y}, such that $Y_t$ remains an $\R^d$-valued random variable for each $t>0$. How the ODE evolves is determined jointly by $\phi_{\mu^{Y_t}}^{\mud}$ (depending on the current law $\mu^{Y_t}$ and $\mud$) and the actual realization of $Y_t$. 

Solving ODE \eqref{Y} is nontrivial. Compared with typical stochastic differential equations (SDEs) with distribution dependence (i.e., McKean-Vlasov SDEs), how \eqref{Y} depends on distributions---via the Kantorovich potential $\phi_{\mu}^{\mud}$---is quite unusual. As $\phi_{\mu}^{\mud}$ is only known to exist generally (without any tractable formula), the regularity of $(y,\mu)\mapsto \nabla \phi_{\mu}^{\mud}(y)$ is largely obscure, which prevents direct application of standard results for McKean-Vlasov SDEs.

In view of this, we will {\it not} deal with ODE \eqref{Y} directly, but focus on the nonlinear Fokker--Planck equation associated with it. By finding a solution $\bm\nu_t$ to the Fokker--Planck equation, we can in turn construct a solution $Y$ to \eqref{Y}, with $\mu^{Y_t}=\bm\nu_t$, by Trevisan's superposition principle \cite{Trevisan16}. This strategy, introduced in Barbu and R\"{o}ckner \cite{BR20}, was recently used in Huang and Zhang \cite{Huang23} for generative modeling under the Jensen-Shannon divergence (JSD) and Huang and Malik \cite{HM25} for non-convex optimization. 

Our first observation is that if the gradient flow $\bm\mu=\{\bm\mu_t\}_{t\ge 0}$ for $J(\cdot):= \frac12 W^2_2(\cdot,\mud)$ satisfies ``$\bm\mu_t\ll \mathcal L^d$ for all $t\ge 0$,'' it will automatically solve the nonlinear Fokker--Planck equation (i.e., \eqref{FP} below). That is, the solvability of \eqref{FP} hinges on the {claim}: ``{\it if we start with $\mu_0\ll\mathcal L^d$, the resulting gradient flow $\bm\mu$ for $J$ must fulfill $\bm\mu_t\ll \mathcal L^d$ for all $t> 0$}.'' This will be proved in a constructive, yet indirect, way. First, we  construct a curve $\bm\beta = \{\bm\beta_s\}_{s\in [0,1]}$ that moves $\bm\beta_0:=\mu_0$ to $\bm\beta_1:=\mud$ at a constant speed; see \eqref{Def: Generalized Geodesic between mu0 and mud} below. By the definition of $\bm\beta$, Lemma~\ref{lem:has densities} finds that ``$\mu_0\ll\mathcal L^d$'' entails ``$\bm\beta_s\ll\mathcal L^d$ for all $s\in [0,1)$'' (excluding $s=1$). By a change of speed from constant to exponentially decreasing, $\{\bm\beta_s\}_{s\in [0,1)}$ is rescaled into $\bm\mu^* = \{\bm\mu^*_t\}_{t\ge 0}$ (i.e., \eqref{Def: mu t} below), so that  $\bm\mu^*_t\ll\mathcal L^d$ for all $t\ge 0$. As $\bm\mu^*$ is by definition a time-changed constant-speed geodesic, the dynamics of $t\mapsto J(\bm\mu^*_t)$ can be decomposed and explicitly computed (Lemma~\ref{lem:CSG}), which reveals that $J(\bm\mu^*_t)$ always decreases in time with a largest possible slope. This ``maximal slope'' behavior in turn implies that $\bm\mu^*$ simply coincides with $\bm\mu$, the gradient flow for $J$ (Theorem~\ref{thm:mu is GF}). The {claim} is thus proved and the explicitly constructed $\bm\mu^*$ solves the nonlinear Fokker--Planck equation \eqref{FP}; 
see Corollary~\ref{coro:sol. to FP}. 

By substituting $\bm\mu^*_t$ for $\mu^{Y_t}$ in \eqref{Y}, we obtain a standard ODE without distribution dependence, i.e., $dY_t = f(t, Y_t) dt$ with $f(t,y) := -\nabla\phi_{\bm\mu^*_t}^{\mud}(y)$. To construct a solution to \eqref{Y}, we aim at finding a process $Y$ such that (i) $t\mapsto Y_t$ satisfies the above standard ODE, and (ii) the law of $Y_t$ coincides with $\bm\mu^*_t$ (i.e., $\mu^{Y_t} = \bm\mu^*_t$) for all $t \ge 0$. Thanks to Trevisan's superposition principle \cite{Trevisan16}, we can choose a probability measure $\P$ on the canonical path space, so that the canonical process $Y_t(\omega) := \omega(t)$ fulfills (i) and (ii) above. This gives a solution $Y$ to \eqref{Y} (Proposition~\ref{thm:sol. to ODE exists}) and it is in fact unique up to time-marginal laws under suitable regularity conditions (Proposition~\ref{prop:unique mu^Y_t}). As $\bm\mu^*$ coincides with the gradient flow for $J(\cdot):= \frac12 W^2_2(\cdot,\mud)$, we conclude that $\mu^{Y_t}=\bm\mu^*_t\to \mud$ exponentially under $W_2$ as $t\to\infty$; see Theorem~\ref{thm:main}, the main theoretic result of this paper. 

Theorem~\ref{thm:main} suggests that we can uncover $\mud$ by simulating ODE \eqref{Y}. To facilitate the actual implementation, we propose a forward Euler scheme (i.e., \eqref{Euler} below), which, for any fixed time step $0<\eps <1$, simulates a sequence of random variables $\{Y^\eps_n\}_{n\in\N\cup\{0\}}$ at time points $\{n\eps\}_{n\in\N\cup\{0\}}$. Interestingly, under our scheme, how the distribution $\mu^{Y^\eps_{n}}$ is transformed into $\mu^{Y^\eps_{n+1}}$ is optimal in the sense of $W_2$ (Lemma~\ref{lem:mu_n is optimal}). This crucially implies that the laws of $\{Y^\eps_n\}_{n\in\N\cup\{0\}}$ {\it all} lie on the curve $\bm\mu^*$, which consists of the time-marginal laws of ODE \eqref{Y} (Proposition~\ref{prop:mu_n on beta}). As $\eps\downarrow 0$, Theorem~\ref{Th: Convergence of Euler discretization} shows that the laws of $\{Y^\eps_n\}_{n\in\N\cup\{0\}}$ gradually cover the entire curve $\bm\mu^*$. 
That is, our Euler scheme does correctly recover the evolving distributions induced by ODE \eqref{Y} in the limit. 

Based on this Euler scheme, Algorithm~\ref{Alg: Simulate W2 GF ODE} (called {W2-FE}, where {FE} means ``forward Euler'') is designed to simulate ODE \eqref{Y}. 
In particular, a ``generator'' $G_\theta:\R^d\to\R^d$, modeled by a deep neural network, is updated at each time point to keep track of the changing distributions along the (discretized) ODE. Remarkably, in updating $G_\theta$, we reuse the same minibatch of points newly generated by the ODE in $K\in\N$ consecutive iterations of stochastic gradient descent (SGD). This is in line with {\it persistent training} in Fischetti et al.\ \cite{Fischetti18}, distinct from standard SGD implementation that takes $K=1$ by default. Note that persistent training fits our ODE approach naturally: as the goal of updating $G_\theta$ is to learn the new distribution induced by the ODE at the next time point, keeping the newly generated points unchanged in SGD iterations allows $G_\theta$ to more accurately represent the new distribution, so that the ODE can be 
 more closely followed. Intriguingly, our algorithm covers the Wasserstein-2 generative adversarial network (W2-GAN) in \cite{Leygonie19} as a special case when $K=1$ (i.e., no persistent training); see Proposition~\ref{prop:equivalence}. 


We implement our algorithm W2-FE with varying persistency levels (i.e., $K$ values) and compare its performance with that of the refined Wasserstein GAN (WGAN) algorithm in \cite{Petzka18} (called W1-LP), one of the most well-performing and stable WGAN algorithms. In both low- and high-dimensional experiments, by increasing the persistency level appropriately, W2-FE converges significantly faster than W1-LP and achieves similar or better outcomes; see Figures~\ref{fig: 2D Qualitative Results}, 
\ref{fig: 2D Quantitative Results}, and \ref{fig: High-D USPS-MNIST 1-NN accuracy}.

The rest of the paper is organized as follows. Section~\ref{subsec:related} reviews some related generative models. Section~\ref{subsec:notation} collects fundamental notation. Section~\ref{sec:preliminaries} 
motivates the gradient-descent ODE to be studied. In Section~\ref{sec:GF}, the gradient flow theory under $W_2$ is briefly reviewed and tailored to the needs of our study. Section~\ref{sec:sol. to FP} finds a solution to the nonlinear Fokker--Planck equation associated with the gradient-descent ODE. Based on this, Section~\ref{sec:sol. to ODE} builds a solution to the ODE, whose time-marginal laws converge to $\mud$ exponentially. A discretized scheme for the ODE is proposed and analyzed in detail in Section~\ref{subsec:Euler}. Section~\ref{sec:W2-FE} designs an algorithm based on the discretized scheme. Its performance is illustrated by numerical experiments in Section~\ref{sec:examples}. 


\subsection{Related Generative Models}\label{subsec:related}
To gradually change an initial distribution $\mu_0$ into a target distribution $\mud$, GANs impose a min-max game between a discriminator and a generator: given the generator's strategy, the discriminator chooses a strategy to maximize an objective; in response to the discriminator's new strategy, the generator updates its strategy to minimize the same objective. Through the two players' recursive optimization, GANs iteratively reduce the discrepancy between $\mu_0$ and $\mud$. There are many  ways to measure the discrepancy (which accordingly yield different objectives); e.g., the vanilla GAN \cite{Goodfellow14} uses JSD, $f$-GAN \cite{f-GAN} uses $f$-divergence, WGAN \cite{Arjovsky17} 
uses the first-order Wasserstein distance (denoted by $W_1$), W2-GAN \cite{Leygonie19} 
uses the $W_2$ distance (the same as this paper), among others. 

A more direct way to reduce the discrepancy between $\mu_0$ and $\mud$ is to move along a gradient flow, i.e., reduce the discrepancy along its negative ``gradient'' (as long as the ``gradient'' is well-defined). 
Interestingly, the recursive min-max optimization of GANs is in fact a (coarse) approximation of an associated gradient flow, as shown in \cite{Huang23} for the vanilla GAN and in this paper for W2-GAN. 
This suggests that a finer approximation of the gradient flow (e.g., achieved by persistent training) will outperform the standard GAN training, which is confirmed numerically in Section~\ref{sec:examples} (where suitably enlarged persistency level improves training performance). 

As WGAN \cite{Arjovsky17} is a popular version of GANs known for its stability, it is natural to ask: (i) What is the underlying gradient flow WGAN attempts to approximate? (ii) What is a better approximation of the gradient flow (which should outperform WGAN)? Recall that under WGAN, the discrepancy between distributions is measured by $W_1$.   
While subdifferential calculus for probability measures is full-fledged under the $p^{th}$-order Wasserstein distance $W_p$ for all $p > 1$ (see e.g., \cite[Chapter 10]{Ambrosio08}), the same construction breaks down when $p = 1$. In particular, the ``Fr\'{e}chet differential'' (\cite[Corollary 10.2.7]{Ambrosio08}) and ``Wasserstein gradient'' (\cite[Definition 5.62]{CD-book-18-I}) are no longer well-defined for $p = 1$. It is then not even clear how ``gradient'' should be defined under $W_1$, let alone answering (i) and (ii). As it stands, WGAN recursively minimizes and maximizes an objective induced by $W_1$---whether it relates to any gradient flow is an interesting question for future research. This is in contrast to our algorithm W2-FE, which directly computes the well-defined gradient flow under $W_2$.

There are well-known methods (e.g., \cite{Courty2017}, \cite{Seguy18}, \cite{Makkuva2019}, \cite{Rout2022}, \cite{Hamri2022}, among others) that directly learn the optimal transport map from $\mu_0$ to $\mud$. They aim to change $\mu_0$ into $\mud$ {\it in one shot}, using the single optimal transport map learned. This is distinct from our gradient-flow approach that changes $\mu_0$ into $\mud$ {\it iteratively}, using numerous intermediate optimal transport maps estimated recursively. It is shown numerically in Section~\ref{sec:examples} that our gradient-flow approach can produce better learning results than the ``one-shot'' methods; see Table~\ref{tab: DA Comparison}.


\subsection{Notation}\label{subsec:notation}
For any Polish space $\mathcal X$, let $\mathcal B(\mathcal X)$ be the Borel $\sigma$-algebra of $\Xc$ and $\mathcal{P}(\mathcal{X})$ be the set of all probability measures on $(\mathcal{X},\mathcal B(\Xc))$. Given two Polish spaces $\mathcal X$ and $\mathcal Y$, $\mu\in\Pc(\Xc)$, and a Borel $f:\mathcal X\to\mathcal Y$, we define $f_\#\mu\in\Pc(\Yc)$, called the pushforward of $\mu$ through $f$, by
$
f_\#\mu(B) := \mu(f^{-1}(B))$  for $B\in \mathcal B(\Yc)$. 
On the product space $\Xc\times\Yc$, consider the projection operators 
$
\pi^1(x,y) := x$ and $\pi^2(x,y) := y$ for all $(x,y)\in\Xc\times\Yc$.
Given $\mu\in\Pc(\Xc)$ and $\nu\in\Pc(\Yc)$, a Borel $\mathbf{t}:\Xc\to\Yc$ is called a transport map from $\mu$ to $\nu$ if $\mathbf{t}_\#\mu=\nu$. Also, a probability $\gamma\in\Pc(\Xc\times\Yc)$ is called a transport plan from $\mu$ to $\nu$ if  $\pi^1_\#\gamma =\mu$ and $\pi^2_\#\gamma = \nu$. Let $\Gamma(\mu,\nu)$ denote the set of all transport plans from $\mu$ to $\nu$. 

Fix any $d\in\N$. Let $\mathbf i$ denote the identity map on $\R^d$ (i.e., $\mathbf i(y):=y$ for all $y\in\R^d$) and $\mathcal L^d$ denote the Lebesgue measure on $\R^d$. For any measure $\mu$ on $\R^d$, we write $\mu\ll\mathcal L^d$ if it is absolutely continuous w.r.t.\ $\mathcal L^d$. Given $p\ge 1$ and $S\subseteq \R^d$, let $L^p(S,\mu)$ be the set of all $f:\R^d\to\R$ with $\int_{S} |f(y)|^p d\mu(y)<\infty$. For the case $\mu=\mathcal L^d$, we will simply write $L^p(S)$ for $L^p(S,\mu)$. 


\section{Preliminaries}\label{sec:preliminaries}
Fix $d\in \N$ and denote by $\mathcal{P}_{2}(\mathbb{R}^{d})$ the set of elements in $\mathcal{P}(\mathbb{R}^{d})$ with finite second moments, i.e., 
\begin{eqnarray*}
    \mathcal{P}_{2}(\mathbb{R}^{d}) := \bigg\{ \mu \in \mathcal{P}(\mathbb{R}^{d}) : \int_{\R^d} |y|^2 d\mu(y) < \infty \bigg\}. 
\end{eqnarray*}
The second-order Wasserstein distance, a metric on $\Pc_2(\R^d)$, is defined by
\begin{eqnarray} \label{Eq: Kantorovich Relaxation}
    W_{2}(\mu, \nu) := \inf_{\gamma \in \Gamma(\mu,\nu)} \left( \int_{\mathbb{R}^{d} \times \mathbb{R}^{d}} |x - y|^{2} \, d\gamma(x, y) \right)^{1/2},\quad \forall \mu,\nu\in \mathcal{P}_{2}(\mathbb{R}^{d}). 
\end{eqnarray}
For $\{\mu_n\}_{n\in\N}$ and $\mu$ in $\Pc_2(\R^d)$ with $W_2(\mu_n,\mu)\to 0$, we will often write $\mu_n\to \mu$ for brevity.  

To state the duality formula for the $W_{2}$ distance, we set $c(x,y):= \frac12|x-y|^2$ for $x,y\in\R^d$ and define the $c$-transform of any $\phi:\R^d\to\R$ by $\phi^c(y) := \inf_{x\in\R^d} \{c(x,y)-\phi(x)\}$, $y\in\R^d.$ 
Also, we say $\phi$ is $c$-concave if $\phi=\psi^c$ for some $\psi:\R^d\to\R$. By \cite[Theorem 5.10 (i)]{Villani09}, the $W_2$ distance fulfills 
\begin{eqnarray} \label{Eq: W2 Duality Formula}
    \frac12 W^2_{2}(\mu, \nu) = \sup \left\{\int_{\R^d} \phi(x) \, d\mu(x) + \int_{\R^d} \phi^c(y) \, d\nu(y): \phi \in L^{1}(\R^d,\mu)\ \hbox{is $c$-concave} \right\}. 
\end{eqnarray}
Thanks to $|x-y|^2\le 2(|x|^2+|y|^2)$ for $x,y\in\R^d$, \cite[Theorem 5.10 (iii)]{Villani09} asserts that a minimizer of  \eqref{Eq: Kantorovich Relaxation} and a maximizer of \eqref{Eq: W2 Duality Formula} always exist.

\begin{definition}\label{def:Kantorovich potential}
Given $\mu,\nu\in\Pc_2(\R^d)$, a $\gamma\in\Gamma(\mu,\nu)$ that minimizes \eqref{Eq: Kantorovich Relaxation} is called an optimal transport plan from $\mu$ to $\nu$. We denote by $\Gamma_0(\mu,\nu)$ the set of all such plans. A $c$-concave $\phi\in L^1(\R^d,\mu)$ that maximizes \eqref{Eq: W2 Duality Formula} is called a Kantorovich potential from $\mu$ to $\nu$ and denoted by $\phi_\mu^\nu$. 
\end{definition}

\begin{remark}\label{rem:nabla u}
The $c$-concavity of $\phi_\mu^\nu$ implies that $f(x):=\frac12|x|^2-\phi_\mu^\nu(x)$, $x\in\R^d$, is convex; see \cite[Section 6.2.3]{Ambrosio08}. It follows that $\nabla f(x)$ (and accordingly $\nabla \phi_\mu^\nu(x)$) exists for $\mathcal L^d$-a.e.\  $x\in\R^d$.   
 \end{remark}

\begin{remark}\label{rem:mu is regular}
For any $\mu,\nu\in\Pc_2(\R^d)$, suppose additionally that $\mu$ belongs to
\begin{eqnarray*}
\Pc_2^r(\R^d) := \{\mu\in\Pc_2(\R^d) : \mu\ll \mathcal L^d\}. 
\end{eqnarray*}
Then, by \cite[Theorem 6.2.4 and Remark 6.2.5 a)]{Ambrosio08}, $\Gamma_0(\mu,\nu)$ is a singleton and the unique optimal transport plan $\gamma\in\Pc(\R^d\times\R^d)$ takes the form
\begin{eqnarray}\label{optimal TP}
\gamma = (\mathbf{i}\times\mathbf{t}_{\mu}^{\nu})_\#\mu
\end{eqnarray}
for some transport map $\mathbf{t}_{\mu}^{\nu}:\R^d\to\R^d$ from $\mu$ to $\nu$, and the map $\mathbf{t}_{\mu}^{\nu}$ can be expressed as
\begin{eqnarray}\label{t and u}
\mathbf{t}_{\mu}^{\nu}(x) = x - \nabla \phi_{\mu}^\nu(x)\quad \hbox{for $\mu$-a.e.\ $x\in\R^d$},
\end{eqnarray}
for any Kantorovich potential $\phi_\mu^\nu$. Note that $\nabla \phi_{\mu}^\nu$ is well-defined $\mu$-a.e.\ by Remark~\ref{rem:nabla u} and $\mu\ll\mathcal L^d$. 
 \end{remark}
 
\begin{remark}\label{rem:optimal TM}
We will call $\mathbf{t}_{\mu}^{\nu}:\R^d\to\R^d$ in \eqref{optimal TP} an optimal transport map from $\mu$ to $\nu$. Indeed, as $\gamma = (\mathbf{i}\times\mathbf{t}_{\mu}^{\nu})_\#\mu$ is an optimal transport plan from $\mu$ to $\nu$, we deduce from \eqref{Eq: Kantorovich Relaxation} that 
\begin{eqnarray*}
W_2(\mu,\nu) = \left(\int_{\mathbb{R}^{d}} |\mathbf{t}_\mu^\nu(x) - x|^{2} \, d\mu(x) \right)^{1/2} = \inf_{\mathbf{t}:\R^d\to\R^d, \mathbf{t}_\#\mu=\nu}\left(\int_{\mathbb{R}^{d}} |\mathbf{t}(x) - x|^{2} \, d\mu(x) \right)^{1/2} .
\end{eqnarray*}
\end{remark}

\subsection{Problem Formulation}
Let $\mud\in\Pc_2(\R^d)$ denote the (unknown) data distribution. Starting with an arbitrary initial guess $\mu_0\in\Pc_2(\R^d)$ of $\mud$, we aim at solving the problem
\begin{eqnarray}\label{to solve}
\min_{\mu\in\Pc_2(\R^d)} W_2^2(\mu,\mud)
\end{eqnarray}
efficiently. By \eqref{Eq: Kantorovich Relaxation}, it can be checked directly that $\mu\mapsto W^2_2(\mu,\mud)$ is convex on $\Pc_2(\R^d)$. Hence, it is natural to ask if \eqref{to solve} can be solved simply by {\it gradient descent}. 
The challenge here is how we should define the geadient of a function on $\Pc_2(\R^d)$. By subdifferential calculus in $\Pc_2(\R^d)$ developed by \cite{Ambrosio08}, for a general $G:\Pc_2(\R^d)\to\R$, we can take the ``gradient of $G$'' to be an element of its Fr\'{e}chet subdifferential, defined below following (10.3.12)-(10.3.13) and Definition 10.3.1 in \cite{Ambrosio08}. 

\begin{definition}\label{def:subdifferential}
Given $G:\Pc_2(\R^d)\to\R$ and $\mu\in\Pc_2(\R^d)$, the Fr\'{e}chet subdifferential of $G$ at $\mu$, denoted by $\partial G(\mu)$, is the collection of all $\xi\in L^2(\R^d,\mu)$ such that  
\begin{eqnarray}\label{subdifferential}
    G(\nu) - G(\mu) \geq \inf_{\gamma\in\Gamma_0(\mu,\nu)}\int_{\mathbb{R}^{d}} \xi(x) \cdot (y - x) \, d\gamma(x,y) + o(W_{2}(\mu, \nu))\quad \forall\nu \in \mathcal{P}_{2}(\mathbb{R}^{d}). 
\end{eqnarray}
\end{definition}

For the function $J:\Pc_2(\R^d)\to\R$ defined by
\begin{eqnarray}\label{J}
J(\mu) := \frac12 W^2_2(\mu,\mud), 
\end{eqnarray}
the proof of \cite[Theorem 10.4.12]{Ambrosio08}, particularly (10.4.53) therein, shows that $\xi\in\partial J(\mu)$ generally exists. The construction of $\xi$, however, relies on an optimal transport plan $\gamma\in \Gamma_0(\mu,\mud)$, which is not unique and does not admit concrete characterizations in general. Fortunately, for the special case where $\mu\in\Pc^r_2(\R^d)$, $\partial J(\mu)$ becomes much more tractable. 

\begin{lemma}\label{lem:subdifferential J}
At any $\mu\in\Pc_2^r(\R^d)$, the Fr\'{e}chet subdifferential of $J:\Pc_2(\R^d)\to\R$ in \eqref{J} is a singleton, i.e.,   
$\partial J (\mu) = \{-(\mathbf{t}_{\mu}^{\mud}-\mathbf i)\} = \{\nabla \phi_\mu^{\mud}\}$, where $\mathbf{t}_{\mu}^{\mud}$ (resp.\ $\phi_\mu^{\mud}$) is an optimal transport map (resp.\ a Kantorovich potential) from $\mu$ to $\mud$. Specifically, if $\xi\in\partial J(\mu)$, then $\xi = \nabla \phi_\mu^{\mud}$ $\mu$-a.e. 
\end{lemma}

\begin{proof}
For any $\nu\in\Pc_2(\R^d)$, by \cite[Corollary 10.2.7]{Ambrosio08} (with $\mu_1\in\Pc^r_2(\R^d)$ and $\mu_2, \mu_3\in\Pc_2(\R^d)$ therein taken to be $\mu\in\Pc_2^r(\R^d)$ and $\mud,\nu\in\Pc_2(\R^d)$, respectively), we have 
\begin{eqnarray*}
J(\nu) - J(\mu) = -\int_{\mathbb{R}^{d}} (\mathbf{t}_{\mu}^{\mud}(x)- x)  \cdot (\mathbf{t}_{\mu}^\nu(x)- x) \, d\mu(x) + o(W_{2}(\mu, \nu)). 
\end{eqnarray*}
In addition, as $\mu\in\Pc_2^r(\R^d)$, $\Gamma_0(\mu,\nu)$ is a singleton that contains only the transport plan in \eqref{optimal TP} (see Remark~\ref{rem:mu is regular}). Hence, the previous equality can be equivalently written as 
\begin{eqnarray}\label{subdifferential J}
J(\nu) - J(\mu) = \inf_{\gamma\in\Gamma_0(\mu,\nu)}\int_{\mathbb{R}^{d}} -(\mathbf{t}_{\mu}^{\mud}(x)- x) \cdot (y- x) \, d\gamma(x,y) + o(W_{2}(\mu, \nu)). 
\end{eqnarray}
That is, $J$ fulfills \eqref{subdifferential} with $\xi=-(\mathbf{t}_{\mu}^{\mud}-\mathbf i)\in\partial J(\mu)$. By \eqref{subdifferential J} and $\Gamma_0(\mu,\nu)$ being a singleton, we may follow the same argument in the proof of \cite[Proposition 5.63]{CD-book-18-I} to show that any $\xi'\in\partial J(\mu)$ must coincide with $-(\mathbf{t}_{\mu}^{\mud}-\mathbf i)$ $\mu$-a.e.. That is, $-(\mathbf{t}_{\mu}^{\mud}-\mathbf i)$ is the unique element in $\partial J(\mu)$. The fact $-(\mathbf{t}_{\mu}^{\mud}-\mathbf i)=\nabla \phi_{\mu}^{\mud}$ $\mu$-a.e.\ follows directly from \eqref{t and u}. 
\end{proof}

\begin{remark}
The proof of Lemma~\ref{lem:subdifferential J} integrates \cite[Corollary 10.2.7]{Ambrosio08} (which implies that $\partial J(\mu)$ contains at least $\xi=-(\mathbf{t}_{\mu}^{\mud}-\mathbf i)$) and \cite[Proposition 5.63]{CD-book-18-I} (which shows the uniqueness of $\xi$). 
\end{remark}

Lemma~\ref{lem:subdifferential J} suggests that \eqref{to solve} can potentially be solved by the {\it gradient-descent} ODE \eqref{Y}, for any $\mu_0\in \Pc_2^r(\R^d)$. 
Note that this ODE is {\it distribution-dependent} in nontrivial ways. At time 0, $Y_0$ is an $\R^d$-valued random variable following $\mu_0\in\Pc_2^r(\R^d)$, an arbitrary initial distribution. This initial randomness trickles through the ODE dynamics, such that $Y_t$ remains an $\R^d$-valued random variable, whose law is  denoted by $\mu^{Y_t}\in\Pc_2(\R^d)$, at each $t>0$. The evolution of the ODE is determined jointly by a Kantorovich potential from the present distribution $\mu^{Y_t}$ to $\mud$ (i.e., the function $\phi_{\mu^{Y_t}}^{\mud}$) and the actual realization of $Y_t$ (which is plugged into the map $y\mapsto\nabla\phi_{\mu^{Y_t}}^{\mud}(y)$).

Our goal is to show that (i) there is a (unique) solution $Y$ to ODE \eqref{Y} and (ii) the law of $Y_t$ will ultimately converge to $\mud$ (i.e., $W_2(\mu^{Y_t},\mud)\to 0$ as $t\to\infty$), thereby solving \eqref{to solve}. How \eqref{Y} depends on distributions---via the Kantorovich potential $\phi_{\mu^{Y_t}}^{\mud}$---poses nontrivial challenges. 
\begin{itemize}
\item [(a)] For the existence of a unique solution to a McKean-Vlasov SDE (i.e., an SDE with dependence on the present law $\mu^{Y_t}$), one needs the SDE's coefficients to be sufficiently regular in both the state variable and $\mu^{Y_t}$. As $\phi_{\mu^{Y_t}}^{\mud}$ in \eqref{Y} is only known to exist (without any tractable formula), the regularity of $(y,\mu)\mapsto \nabla \phi_{\mu}^{\mud}(y)$ is left largely unexplored. 
\item [(b)] The coefficient $-\nabla \phi_{\mu^{Y_t}}^{\mud}(\cdot)$ in \eqref{Y} is only well-defined $\mathcal L^d$-a.e.\ (Remark~\ref{rem:nabla u}). If the random variable $Y_t$ takes a value where $-\nabla \phi_{\mu^{Y_t}}^{\mud}(\cdot)$ is undefined, the ODE dynamics will simply break down. To avoid this, we need ``$\mu^{Y_t}\ll\mathcal{L}^d$ for all $t\ge 0$,'' for every $Y_t$ to concentrate on where $-\nabla \phi_{\mu^{Y_t}}^{\mud}(\cdot)$ is well-defined. At $t=0$, this is achieved by choosing $\mu_0$ directly from $\Pc_2^r(\R^d)$. 
The question is whether subsequent distributions $\{\mu^{Y_t}\}_{t>0}$ will remain in $\Pc_2^r(\R^d)$. 

\item [(c)] In fact, ODE \eqref{Y} demands much more than ``$\mu^{Y_t}\ll\mathcal{L}^d$ for all $t\ge 0$.'' Given a stochastic process $Y$ (defined on some probability space $(\Omega,\mathcal F,\P)$) satisfying ``$\mu^{Y_t}\ll\mathcal{L}^d$ for all $t\ge 0$,'' we know from (b) that ``$- \nabla \phi_{\mu^{Y_t}}^{\mud}(Y_t(\omega))$ is well-defined $\P$-a.s., for each fixed $t\ge 0$.'' However, to discuss if $Y$ fulfills ODE \eqref{Y}, we need at least the stronger condition ``$-\nabla \phi_{\mu^{Y_t}}^{\mud}(Y_t(\omega))$ is well-defined for all $t\ge 0$ and $t\mapsto -\nabla \phi_{\mu^{Y_t}}^{\mud}(Y_t(\omega))$ is measurable, $\P$-a.s.'', so that the integrals $\int_0^t -\nabla \phi_{\mu^{Y_s}}^{\mud}(Y_s(\omega)) ds$, $\forall t\ge 0$, from the ODE can be well-defined $\P$-a.s. 
\end{itemize}
Due to these challenges, we make two specific moves. First, instead of finding a process $Y$ satisfying the dynamics \eqref{Y}, we look for a pair $(Y,v)$, with $v:(0,\infty)\times\R^d\to\R^d$ Borel, that fulfill
\begin{equation}\label{new Y}
dY_t = v(t, Y_t) dt,\quad \mu^{Y_0} = \mu_0\in\Pc_2^r(\R^d);\qquad \hbox{$v(t,\cdot) = -\nabla \phi_{\mu^{Y_t}}^{\mud}(\cdot)$\  \ $\mu^{Y_t}$-a.e.}\ \ \forall t> 0. 
\end{equation}
The technical advantage of \eqref{new Y} over \eqref{Y} is as follows: as $(t,y)\mapsto v(t,y)$ is a well-defined Borel map, working with $dY_t = v(t, Y_t) dt$ can largely circumvent challenges (b) and (c) above, while the property ``$v(t,\cdot) = -\nabla \phi_{\mu^{Y_t}}^{\mud}(\cdot)$ $\mu^{Y_t}$-a.e.\ $\forall t> 0$'' keeps the ``gradient descent'' interpretation. Second, to circumvent challenge (a), we will {\it not} work with ODE \eqref{new Y} directly, but instead focus on the corresponding nonlinear Fokker--Planck equation (or, continuity equation) 
\begin{eqnarray}\label{FP}
\partial_t \bm\mu_t  + \Div (v_t \bm\mu_{t}) = 0,\quad \lim_{t\downarrow 0}\bm\mu_t = \mu_0\in\Pc_2^r(\R^d);\qquad  v_t(\cdot) =-\nabla \phi_{\bm\mu_t}^{\mud}(\cdot)\ \ \bm\mu_t\hbox{-a.e.}\ \ \forall t>0,
\end{eqnarray}
where we write $v_t(\cdot)$ for $v(t,\cdot)$. By first finding a solution $\{\bm\mu_t\}_{t\ge 0}$ to \eqref{FP}, we can in turn construct a solution to ODE \eqref{new Y}, made specifically from $\{\bm\mu_t\}_{t\ge 0}$. This way of solving a distribution-dependent ODE (or SDE), introduced in Barbu and R\"{o}ckner \cite{BR20}, was recently used in Huang and Zhang \cite{Huang23} for a learning problem similar to \eqref{to solve}, with $W_2$ replaced by the Jensen-Shannon divergence (JSD), and also in Huang and Malik \cite{HM25} for non-convex optimization. 
To actually find a solution to \eqref{FP}, we will rely on ``gradient flows'' in $\Pc_2(\R^d)$, which are now introduced. 


\section{Gradient Flows in $\Pc_2(\R^d)$}\label{sec:GF}
In this section, key aspects of the gradient flow theory for probability measures will be introduced from \cite{Ambrosio08} and  tailored to the needs of Sections~\ref{sec:sol. to FP} and \ref{sec:sol. to ODE} below. 
To begin with, we introduce absolutely continuous curves in $\Pc_2(\R^d)$ in line with \cite[Definition 1.1.1]{Ambrosio08}. 

\begin{definition}\label{def:AC curves}
Fix an open interval $I\subseteq(0,\infty)$. For any $p\ge 1$, we say $\bm{\mu}:I\to \Pc_2(\R^d)$ belongs to $AC^p(I;\Pc_2(\R^d))$ (resp.\ $AC^p_{\operatorname{loc}}(I;\Pc_2(\R^d))$) if there exists $m\in L^p(I)$ (resp.\ $m\in L^p_{\operatorname{loc}}(I)$) such that 
\begin{eqnarray}\label{AC curve}
W_{2}(\bm\mu_{t_{1}}, \bm\mu_{{t_{2}}}) \leq \int_{t_{1}}^{t_{2}} m(t) \, dt, \quad \forall  t_{1},t_{2}\in I\ \hbox{with}\ t_1\le t_2.
\end{eqnarray}
For the case $p=1$, we call $\mu$ an absolutely continuous (resp.\ locally absolutely continuous) curve and simply write $\bm\mu\in AC(I;\Pc_2(\R^d))$ (resp.\ $\bm\mu\in AC_{\operatorname{loc}}(I;\Pc_2(\R^d))$). 
\end{definition}

\begin{remark}\label{rem:AC inclusions}
Fix any $p\ge 1$ and open interval $I\subseteq (0,\infty)$. 
\begin{itemize}
\item [(i)] If $\mathcal L^1(I)<\infty$, since $L^p(I)\subseteq L^1(I)$, the inclusion $AC^p(I;\Pc_2(\R^d))\subseteq AC(I;\Pc_2(\R^d))$ follows; similarly, we also have $AC_{\operatorname{loc}}^p(I;\Pc_2(\R^d))\subseteq AC_{\operatorname{loc}}(I;\Pc_2(\R^d))$. 
\item [(ii)] If $\mathcal L^1(I)=\infty$, the inclusions in (i) no longer hold in general. Nonetheless, for any open interval $I'\subseteq I$  with $\mathcal L^1(I')<\infty$, since $L^p_{\operatorname{loc}}(I)\subseteq L^p(I')\subseteq L^1(I')$ , we have $AC_{\operatorname{loc}}^p(I;\Pc_2(\R^d))\subseteq AC^p(I';\Pc_2(\R^d))\subseteq AC(I';\Pc_2(\R^d))$. 
\end{itemize}
\end{remark}

\begin{remark}\label{rem:mu'}
Fix any $p\ge 1$ and open interval $I\subseteq (0,\infty)$. For any $\bm\mu\in AC^p(I;\Pc_2(\R^d))$ (resp.\ $\bm\mu\in AC^p_{\operatorname{loc}}(I;\Pc_2(\R^d))$), as a direct consequence of \eqref{AC curve}, the so-called ``{\it metric derivative}''
\begin{eqnarray}\label{MD}
|\bm\mu'|(t) := \lim_{s\to t} \frac{W_2(\bm\mu_s,\bm\mu_t)}{|s-t|} 
\end{eqnarray}
exists with $|\bm\mu'|(t)\le m(t)$ for $\mathcal L^1$-a.e. $t\in I$, which in turn implies $|\bm\mu'|\in L^p(I)$ (resp.\ $|\bm\mu'|\in L^p_{\operatorname{loc}}(I)$); see \cite[Theorem 1.1.2]{Ambrosio08} for detailed arguments. 
\end{remark}

The next result shows that every $\bm\mu\in AC^2_{\operatorname{loc}}((0,\infty);\Pc_2(\R^d))$ is inherently associated with a vector field $v:(0,\infty)\times\R^d\to \R^d$, which will be understood as the ``velocity'' of $\bm\mu$.

\begin{lemma}\label{lem:velocity exists}
For any $\bm\mu\in AC^2_{\operatorname{loc}}((0,\infty);\Pc_2(\R^d))$, there exists a unique Borel vector field $v:(0,\infty)\times\R^d\to \R^d$, written $(t,x)\mapsto v_t(x)$, such that 
\begin{eqnarray}\label{v conditions}
v_t \in L^2(\R^d,\bm\mu_t)\quad\hbox{and}\quad \|v_t\|_{L^2(\R^d,\bm\mu_t)} = |\bm\mu'|(t)\quad \hbox{for}\ \mathcal L^1\hbox{-a.e.\ $t > 0$}
\end{eqnarray}
and the resulting Fokker--Planck equation
\begin{eqnarray}\label{FP'}
\partial_t \bm\mu_t  + \Div (v_t \bm\mu_{t}) = 0\quad \hbox{on}\quad (0,\infty)\times\R^d
\end{eqnarray}
holds in the sense of distributions, i.e., 
\begin{eqnarray}\label{FP with v}
\int_0^\infty \int_{\R^d} \left(\partial_t\varphi(t,x) + v_t(x)\cdot\nabla\varphi(t,x)\right) d\bm\mu_t(x) dt =0\quad \forall\varphi\in C^\infty_c((0,\infty)\times\R^d).
\end{eqnarray}
\end{lemma}

\begin{remark}
Lemma~\ref{lem:velocity exists} closely resembles \cite[Theorem 8.3.1]{Ambrosio08}, but it requires only the condition $\bm\mu\in AC^2_{\operatorname{loc}}((0,\infty);\Pc_2(\R^d))$, weaker than $\bm\mu\in AC((0,\infty);\Pc_2(\R^d))$ assumed in \cite[Theorem 8.3.1]{Ambrosio08}. As the proof below demonstrates, to accommodate this weaker condition, we need to first focus on bounded intervals $\{I_n\}_{n\in\N}$ where $\bm\mu\in AC(I_n;\Pc_2(\R^d))$ holds for all $n\in\N$. This allows \cite[Theorem 8.3.1]{Ambrosio08} to be applied to $I_n$, so that \eqref{v conditions} and \eqref{FP'} hold for $t\in I_n$ (i.e., \eqref{v^n conditions} and \eqref{FP with v^n} below). By suitably piecing the results together for all $n\in\N$, 
we obtain \eqref{v conditions} and \eqref{FP'}. 
\end{remark}

\begin{proof}
Take an increasing sequence $\{I_n\}_{n\in\N}$ of open intervals such that $\mathcal L^1(I_n)<\infty$ and $\bigcup_{n\in\N} I_n=(0,\infty)$. With $\bm\mu\in AC^2_{\operatorname{loc}}((0,\infty);\Pc_2(\R^d))$, Remark~\ref{rem:AC inclusions} (ii) implies $\bm\mu\in AC(I_n;\Pc_2(\R^d))$ for all $n\in\N$. We can then apply \cite[Theorem 8.3.1]{Ambrosio08} to $\bm\mu$ on each interval $I_n$, which asserts the existence of a Borel vector field $v^n: I_n\times\R^d \to \R^d$ such that 
\begin{eqnarray}\label{v^n conditions}
v^n_t \in L^2(\R^d,\bm\mu_t)\quad\hbox{and}\quad \|v^n_t\|_{L^2(\R^d,\bm\mu_t)} = |\bm\mu'|(t)\quad \hbox{for $\mathcal L^1$-a.e.\ $t\in I_n$}\ 
\end{eqnarray}
and $\partial_t \bm\mu_t  - \Div (v^n_t \bm\mu_{t}) = 0$ holds on $I_n\times\R^d$ in the sense of distributions, i.e., 
\begin{eqnarray}\label{FP with v^n}
\int_{I_n} \int_{\R^d} \left(\partial_t\varphi(t,x) + v^n_t(x)\cdot\nabla\varphi(t,x)\right) d\bm\mu_t(x) dt =0\quad \forall\varphi\in C^\infty_c(I_n\times\R^d). 
\end{eqnarray}
Moreover, by \cite[Proposition 8.4.5]{Ambrosio08}, $v^n_t:\R^d\to\R^d$ is uniquely determined $\mathcal L^1$-a.e.\ in $I_n$ by \eqref{v^n conditions} and \eqref{FP with v^n}. By such uniqueness on each $I_n$, the function $v_t(x) : = v^n_t(x)$, for $(t,x)\in I_n\times\R^d$ and $n\in\N$, is well-defined and Borel. Then, \eqref{v conditions} and \eqref{FP with v} follow directly from \eqref{v^n conditions} and \eqref{FP with v^n}.
\end{proof}


\begin{definition}\label{def:velocity}
Given $\bm\mu\in AC^2_{\operatorname{loc}}((0,\infty);\Pc_2(\R^d))$, the unique Borel vector field $v:(0,\infty)\times\R^d\to\R^d$ that fulfills \eqref{v conditions} and \eqref{FP'} is called the velocity field of $\bm\mu$.  
\end{definition}

\begin{remark}
Note that \eqref{v conditions} and \eqref{FP'} 
allow the use of Trevisan's superposition principle \cite{Trevisan16}, which gives a weak solution to the following ODE: 
$dY_t = v_t(Y_t) dt$ with $\mu^{Y_t} = \bm\mu_t$ for all $t\in(0,\infty)$; 
see the precise arguments in the proof of Theorem~\ref{thm:sol. to ODE exists} below. That is, at every $t>0$, the vector $v_t(Y_t)$ dictates how $Y_t$ moves in $\R^d$ instantaneously, and the resulting laws of the random variables $\{Y_t\}_{t>0}$ recover $\bm\mu\in AC^2_{\operatorname{loc}}((0,\infty);\Pc_2(\R^d))$. This explains why we call $v$ the velocity field of $\bm\mu$. 
\end{remark}

If the velocity field $v$ always guides $\bm\mu$ forward along the ``negative gradient'' of some function $G:\Pc_2(\R^d)\to\R$, 
we naturally call $\bm\mu$ a ``gradient flow.'' 

\begin{definition}\label{def:GF}
We say $\bm\mu\in AC^2_{\operatorname{loc}}((0,\infty);\Pc_2(\R^d))$ is a gradient flow for $G:\Pc_2(\R^d)\to\R$ if its  velocity field $v:(0,\infty)\times\R^d\to\R^d$ satisfies $v_t\in -\partial G(\bm\mu_t)$ for $\mathcal L^1$-a.e.\ $t>0$.  
\end{definition}
Definition~\ref{def:GF} is in line with \cite[Definition 11.1.1]{Ambrosio08}. While the latter seems to require additionally that $v_t$ lies in the tangent bundle to $\Pc_2(\R^d)$ at $\bm\mu_t$, defined as the $L^2(\R^d,\bm\mu_t)$-closure of $\{\nabla \psi: \psi\in C^\infty_c(\R^d)\}$, $v_t$ readily fulfills this by \cite[Proposition 8.4.5]{Ambrosio08} and the proof of Lemma~\ref{lem:velocity exists} above.

For $J:\Pc_2(\R^d)\to\R$ in \eqref{J}, a gradient flow $\bm\mu$ does exist and is unique once an initial measure $\mu_0\in\Pc_2(\R^d)$ is specified. Moreover, it diminishes $J$ exponentially fast.  

\begin{proposition}\label{prop:GF exists}
For any $\mu_0\in\Pc_2(\R^d)$, there exists a gradient flow $\bm\mu\in AC^2_{\operatorname{loc}}((0,\infty);\Pc_2(\R^d))$ for $J:\Pc_2(\R^d)\to\R$ in \eqref{J} with $\lim_{t\downarrow 0}\bm\mu_t=\mu_0$. Moreover, 
\begin{align}
J(\bm\mu_t) &\le e^{-2t} J(\mu_0)\quad \forall t>0;\label{exponentially decreasing}\\
J(\bm\mu_t)-J(\bm\mu_s) &= -\int_s^t \int_{\R^d}|v_u(x)|^2 d\bm\mu_u(x) du\le 0\quad\forall 0\le s< t, \label{energy identity}
\end{align}
where $v:(0,\infty)\times\R^d$ is the velocity field of $\bm\mu$. If $\bar{\bm\mu}\in AC^2_{\operatorname{loc}}((0,\infty);\Pc_2(\R^d))$ is also a gradient flow for $J$ with $\lim_{t\downarrow 0}\bar{\bm\mu}_t=\mu_0$, then $W_2(\bm\mu_t,\bar{\bm\mu}_t) = 0$ for all $t>0$. 
\end{proposition}

Proposition~\ref{prop:GF exists} stems from classical existence and uniqueness results for gradient flows in $\Pc_2(\R^d)$. Specifically, as $J:\Pc_2(\R^d)\to\R$ in \eqref{J} is $1$-convex along generalized geodesics (see e.g., \cite[Lemma 9.2.1]{Ambrosio08}), we can directly apply \cite[Theorem 11.2.1]{Ambrosio08} to conclude that there exists a gradient flow for $J$ with $\lim_{t\downarrow 0}{\bm\mu}_t=\mu_0$ and it fulfills \eqref{exponentially decreasing} and \eqref{energy identity} (which correspond to (11.2.5) and (11.2.4) in \cite{Ambrosio08}, respectively); the uniqueness part of Proposition~\ref{prop:GF exists} follows from \cite[Theorem 11.1.4]{Ambrosio08}. 


\begin{remark}
Proposition~\ref{prop:GF exists} shows the advantage of working with $W_2$ in \eqref{to solve}: not only does a unique gradient flow exists, it converges to $\mud$ exponentially fast, thanks to the full-fledged gradient flow theory under $W_2$. This is in contrast to Huang and Zhang \cite{Huang23}, where \eqref{to solve} is studied with the $W_2$ distance replaced by JSD. As no gradient flow theory under JSD is known, a substantial effort is spent in \cite{Huang23} to define a proper ``gradient'' notion and show that the associated gradient flow does exist and converge to $\mud$, relying on the theory of differential equations on Banach spaces. 
\end{remark}

\begin{remark}\label{rem:how to simulate?}
Despite the desirable properties in Proposition~\ref{prop:GF exists}, it is unclear how the gradient flow for $J$ in \eqref{J} can actually be computed. The construction of the gradient flow in \cite[Theorem 11.2.1]{Ambrosio08} is based on a discretized scheme that requires an optimization problem to be solved at each time step. Although \cite[Corollary 2.2.2]{Ambrosio08} asserts that solutions to these problems generally exist, it provides no clear clues for actually finding the solutions and thereby implementing the discretized scheme. 
\end{remark}

A gradient flow for $J:\Pc_2(\R^d)\to\R$ in \eqref{J} can be characterized equivalently as a curve of ``maximal slope.'' 
Specifically, consider the local slope of $J$ at $\nu\in\Pc_2(\R^d)$, defined by 
\begin{eqnarray}\label{LS}
|\partial J|(\nu) := \limsup_{\rho\to\nu} \frac{(J(\nu)-J(\rho))^+}{W_2(\nu, \rho)}.
\end{eqnarray}
Now, for any $\bm\mu\in AC^2_{\operatorname{loc}}((0,\infty);\Pc_2(\R^d))$, observe that
\begin{eqnarray*}
\frac{d}{dt} J(\bm\mu_t) &= \lim_{h\to 0}\frac{J(\bm\mu_{t+h})-J(\bm\mu_t)}{W_2(\bm\mu_{t+h},\bm\mu_t)} \frac{W_2(\bm\mu_{t+h},\bm\mu_t)}{h}\ge - |\partial J|(\bm\mu_t) |\bm\mu'|(t)\quad\forall t>0,
\end{eqnarray*}
where the inequality follows directly from the definitions of $|\partial J|$ and $|\bm\mu'|$ in \eqref{LS} and \eqref{MD}, respectively. That is, $|\partial J|(\bm\mu_t) |\bm\mu'|(t)$ is the largest possible slope (in magnitude) of $s\mapsto J(\bm\mu_s)$ at time $t$. By Young's inequality, the above relation yields
\begin{eqnarray*}
\frac{d}{dt} J(\bm\mu_t) \ge - |\partial J|(\bm\mu_t) |\bm\mu'|(t)\ge -\frac{1}{2}\left(|\partial J|^2(\bm\mu_t) + |\bm\mu'|^2(t)\right)\quad\forall t>0.
\end{eqnarray*}
Hence, if it can be shown that
\begin{eqnarray}\label{max slope}
\frac{d}{dt} J(\bm\mu_t) \le -\frac{1}{2}\left(|\partial J|^2(\bm\mu_t) + |\bm\mu'|^2(t)\right)\quad\hbox{for $\mathcal L^1$-a.e.\ $t>0$},
\end{eqnarray}
we get $\frac{d}{dt} J(\bm\mu_t) = - |\partial J|(\bm\mu_t) |\bm\mu'|(t)$ for $\mathcal L^d$-a.e.\ $t>0$, i.e., $s\mapsto J(\bm\mu_s)$ almost always moves with the largest possible slope. This behavior characterizes gradient flows for $J$.

\begin{proposition}\label{prop:max slope = GF}
Given $\bm\mu\in AC^2_{\operatorname{loc}}((0,\infty);\Pc_2(\R^d))$, $\bm\mu$ satisfies \eqref{max slope} if and only if $\bm\mu$ is a gradient flow for $J$ as in Definition~\ref{def:GF}. 
\end{proposition}

Proposition~\ref{prop:max slope = GF} essentially restates \cite[Theorem 11.1.3]{Ambrosio08} in our present context. What requires a careful explanation is how \cite[Theorem 11.1.3]{Ambrosio08} can be applied here. By \cite[Lemma 9.2.1 and Remark 9.2.8]{Ambrosio08}, $J:\Pc_2(\R^d)\to\R$ in \eqref{J} is $1$-geodesically convex (see \cite[Definition 9.1.1]{Ambrosio08}). This particularly allows the use of \cite[Lemma 10.3.8]{Ambrosio08} and \cite[Theorem 4.1.2 (i)]{Ambrosio08}, which assert that $J$ is a regular functional (see \cite[Definition 10.3.9]{Ambrosio08}) and a discretized minimization problem for $J$ has a solution. Under these two properties of $J$, \cite[Theorem 11.1.3]{Ambrosio08} is applicable and it states that $\bm\mu$ satisfies \eqref{max slope} if and only if it is a gradient flow for $J$ and $t\mapsto J(\bm\mu_t)$ equals a function of bounded variation for $\mathcal L^1$-a.e.\ $t>0$. Finally, by Proposition~\ref{prop:GF exists}, for any gradient flow $\bm\mu$ for $J$, $t\mapsto J(\bm\mu_t)$ is nonincreasing (see \eqref{energy identity}) and thus of bounded variation. We therefore obtain Proposition~\ref{prop:max slope = GF}. 



\section{Solving the Fokker--Planck Equation \eqref{FP}}\label{sec:sol. to FP}
For any $\mu_0\in\Pc_2(\R^d)$, in view of Proposition~\ref{prop:GF exists} and Definitions~\ref{def:GF} and \ref{def:velocity}, there is a solution $\{\bm\mu_t\}_{t>0}$ to the Fokker--Planck equation \eqref{FP'} with $\lim_{t\downarrow 0} \bm\mu_t = \mu_0$ and its  velocity field $v_t$ lies in the negative Fr\'{e}chet subdifferential $-\partial J(\bm\mu_t)$ for all $t>0$. A key observation is that if we additionally have ``$\bm\mu_t\ll \mathcal{L}^d$ for all $t\ge 0$,'' as $\partial J(\bm\mu_t)$ in this case reduces to a singleton containing only $\nabla \phi_{\bm\mu_t}^{\mud}$ (Lemma~\ref{lem:subdifferential J}),  $v_t$ must equal $-\nabla \phi_{\bm\mu_t}^{\mud}$ $\bm\mu_t$-a.e.\ for all $t\ge 0$. This turns the general Fokker--Planck equation \eqref{FP'} into the one \eqref{FP} we aim to solve. Hence, the problem boils down to: If we take $\mu_0\in\Pc_2^r(\R^d)$, does the gradient flow $\bm\mu$ in Proposition~\ref{prop:GF exists} fulfill $\bm\mu_t\in\Pc_2^r(\R^d)$ for all $t>0$?

In this section, we will provide an affirmative answer to this. First, for any $\mu_0\in\Pc_2(\R^d)$, we will construct a curve in $\Pc_2(\R^d)$ that moves $\mu_0$ to $\mud$ at a constant speed. By its explicit construction, we find that as long as $\mu_0\in\Pc^r_2(\R^d)$, the entire curve will remain in $\Pc_2^r(\R^d)$, except possibly  the end point $\mud\in\Pc_2(\R^d)$; see Lemma~\ref{lem:has densities}. Next, by a suitable change of speed, the rescaled curve will exactly coincide with the gradient flow in Proposition~\ref{prop:GF exists}; see Theorem~\ref{thm:mu is GF}.  

To begin with, for any $\mu_0\in\Pc_2(\R^d)$, take $\gamma\in\Gamma_0(\mu_0,\mud)$ and consider 
\begin{eqnarray} \label{Def: Generalized Geodesic between mu0 and mud}
\bm\beta_s := \left((1-s)\pi^1 + s \pi^2\right)_{\#}\gamma\in\Pc_2(\R^d), \quad s \in [0,1].
\end{eqnarray}

\begin{remark}
By \cite[Theorem 7.2.2]{Ambrosio08}, 
 $\{\bm\beta_s\}_{s\in[0,1]}$ in \eqref{Def: Generalized Geodesic between mu0 and mud} is a constant-speed geodesic in $\Pc_2(\R^d)$ connecting $\mu_0$ and $\mud$, i.e., it satisfies, per \cite[Definition 2.4.2]{Ambrosio08}, 
\begin{eqnarray} \label{Eq: s1 s2 constant speed geodesic formula}
   W_{2}(\bm\beta_{s_{1}}, \bm\beta_{s_{2}}) = (s_{2} - s_{1}) W_{2}(\mu_{0}, \mud), \quad \forall 0 \leq s_{1} \leq s_{2} \leq 1.
\end{eqnarray}
\end{remark}

To change the speed of $\bm\beta$ from constant to exponentially decreasing, we consider 
\begin{eqnarray} \label{Def: mu t}
    \bm\mu^*_{t} := \bm\beta_{1-e^{-t}} = \left(e^{-t} \pi^1 + (1 - e^{-t})\pi^2\right)_{\#}\gamma,\quad t\in [0,\infty). 
\end{eqnarray}


\begin{lemma}\label{lem:CSG}
For any $\mu_0\in\Pc_2(\R^d)$, $\bm\mu^*\in AC^{p}((0, \infty); \mathcal{P}_{2}(\mathbb{R}^{d}))$ for all $p\ge 1$. Moreover, 
\begin{eqnarray*}
|(\bm\mu^*)'|(t) = e^{-t} W_2(\mu_0,\mud)\quad \hbox{and}\quad \frac{d}{dt} J(\bm\mu^*_{t}) = -e^{-2t} W_2^2(\mu_0,\mud)\qquad \forall t>0,
\end{eqnarray*}
where $|(\bm\mu^*)'|(t)$ is the metric derivative defined in \eqref{MD} and $J:\Pc_2(\R^d)\to\R^d$ is given by \eqref{J}. 
\end{lemma}

\begin{proof}
For any $0 < t_{1} \leq t_{2} < \infty$, by the definition of $\bm\mu^*$ and \eqref{Eq: s1 s2 constant speed geodesic formula},     
    \begin{eqnarray} \label{Eq: mut is AC curve}
        \begin{aligned}
                    W_{2}(\bm\mu^*_{t_{1}}, \bm\mu^*_{t_{2}}) &= W_{2}\left(\bm\beta_{1-e^{-t_{1}}}, \bm\beta_{1-e^{-t_{2}}}\right) \\
                    &= -(e^{-t_{2}} - e^{-t_{1}}) W_{2}(\mu_{0}, \mud)= \int_{t_{1}}^{t_{2}} e^{-t} W_{2}(\mu_{0}, \mu^{d}) \, dt.
        \end{aligned}
    \end{eqnarray}
As $e^{-t} \in L^{p}(0, \infty)$ for all $p\ge 1$, we conclude $\bm\mu^*\in AC^{p}((0, \infty); \mathcal{P}_{2}(\mathbb{R}^{d}))$ for all $p\ge 1$; recall Definition~\ref{def:AC curves}. For any $t>0$, by \eqref{MD} and using \eqref{Eq: s1 s2 constant speed geodesic formula} in a similar way, 
    \begin{eqnarray*}
        \begin{aligned}
             |(\bm\mu^*)'|(t)=\lim_{t' \rightarrow t} \frac{W_{2}(\bm\mu^*_{t'}, \bm\mu^*_{t})}{|t'-t|}   &= \lim_{t' \rightarrow t} \frac{|e^{-t'} - e^{-t}|W_{2}(\mu_{0}, \mud)}{|t' - t|} \\
            &= W_{2}(\mu_{0}, \mud) \left|\Big(\frac{d}{ds} e^{-s}\Big)\Big|_{s=t} \right|= e^{-t} W_{2}(\mu_{0}, \mud).
        \end{aligned}
    \end{eqnarray*}
On the other hand, 
    \begin{eqnarray}
        \begin{aligned}
            \frac{d}{dt} J(\bm\mu^*_{t}) = \frac{1}{2} \frac{d}{dt} W_{2}^{2}(\bm\mu^*_{t}, \mud) &= \frac{1}{2} \lim_{h \rightarrow 0} \frac{W_{2}^{2}(\bm\mu^*_{t+h}, \mud) - W_{2}^{2}(\bm\mu^*_{t}, \mud)}{h} \\
            &= \frac{1}{2} \lim_{h \rightarrow 0} \frac{e^{-2(t+h)} W_{2}^{2}(\mu_{0}, \mud) - e^{-2t}W_{2}^{2}(\mu_{0}, \mud)}{h} \\
            &= \frac{1}{2}W_{2}^{2}(\mu_{0}, \mud) \Big(\frac{d}{ds} e^{-2s}\Big)\Big|_{s=t} = -e^{-2t}W_{2}^{2}(\mu_{0}, \mud),
        \end{aligned}
    \end{eqnarray}
where the second line follows from \eqref{Def: mu t}, $\bm\beta_1=\mud$, and \eqref{Eq: s1 s2 constant speed geodesic formula}. 
\end{proof}

Now, let us assume additionally that $\mu_0\ll\mathcal{L}^d$ (i.e., $\mu_0\in\Pc^r_2(\R^d)$). By Remarks~\ref{rem:mu is regular} and \ref{rem:optimal TM}, an optimal transport map $\mathbf{t}_{\mu_0}^{\mud}:\R^d\to\R^d$ from $\mu_0$ to $\mud$ exists and $\gamma= (\mathbf{i}\times\mathbf{t}_{\mu_0}^{\mud})_\#\mu_0$ is the unique element in $\Gamma_0(\mu_0,\mud)$. Hence, $\{\bm\beta_s\}_{s\in[0,1]}$ in \eqref{Def: Generalized Geodesic between mu0 and mud} now takes the more concrete form
\begin{eqnarray} \label{beta'}
    \bm\beta_s = \left((1-s)\mathbf{i} + s \mathbf{t}_{\mu_0}^{\mud}\right)_{\#}\mu_{0}\in\Pc_2(\R^d), \quad s \in [0,1].
\end{eqnarray}
Furthermore, the entire curve $\bm\beta$ now lies in $\Pc_2^r(\R^d)$, except possibly the right endpoint. 


\begin{lemma}\label{lem:has densities}
As long as $\mu_0\in\Pc^r_2(\R^d)$, we have $\bm\beta_s\in\Pc^r_2(\R^d)$ for all $s\in[0,1)$. 
\end{lemma}

\begin{proof}
For any fixed $s\in(0,1)$, thanks to \eqref{beta'} and \eqref{t and u}, $\bm\beta_s$ can be equivalently expressed as 
\begin{eqnarray}\label{gamma formula}
\bm\beta_s 
= (\mathbf{i} - s \nabla \phi_{\mu_0}^{\mud})_{\#}\mu_{0}.
\end{eqnarray}
Recall from Remark~\ref{rem:nabla u} that $f(x):=\frac12|x|^2-\phi_{\mu_0}^{\mud}(x)$, $x\in\R^d$, is convex. It follows that $f_s(x):=\frac12|x|^2-s\phi_{\mu_0}^{\mud}(x)$, $x\in\R^d$, is strictly (in fact, uniformly) convex. Indeed, $f_s$ can be expressed as
\begin{eqnarray*}
f_s(x) = \frac12|x|^2-s\bigg(\frac12|x|^2 - f(x)\bigg)=\frac12(1-s)|x|^2+s f(x), 
\end{eqnarray*}
where $\frac12(1-s)|x|^2$ is uniformly convex and $s f(x)$ is convex. Hence, by \cite[Theorem 5.5.4]{Ambrosio08}, there exists a $\mathcal L^d$-negligible $S\subset \R^d$ such that $\nabla f_s$ exists and is injective with $|\operatorname{det}\nabla^2 f_s|>0$ on $\R^d\setminus S$. We can then apply \cite[Lemma 5.5.3]{Ambrosio08} to conclude that $(\nabla f_s)_\#\mu_0\ll \mathcal L^d$. Finally, by the definition of $f_s$ and Remark~\ref{rem:nabla u}, $\nabla f_s = \mathbf{i}-s\nabla\phi_{\mu_0}^{\mud}$ $\mu_0$-a.e. This, together with \eqref{gamma formula}, implies $\bm\beta_s=(\nabla f_s)_\#\mu_0\ll \mathcal L^d$.  
\end{proof}

For any $\mu_0\in\Pc^r_2(\R^d)$, thanks to the concrete form of $\bm\beta$ in \eqref{beta'}, $\bm\mu^*$ defined in \eqref{Def: mu t} becomes 
\begin{eqnarray}\label{mu^*'}
\bm\mu^*_{t} = \bm\beta_{1-e^{-t}} = \left(e^{-t} \mathbf i + (1 - e^{-t})\mathbf t_{\mu_0}^{\mud}\right)_{\#}\mu_0,\quad t\in [0,\infty). 
\end{eqnarray}
In fact, this exactly coincides with the unique gradient flow for $J:\Pc_2(\R^d)\to\R$ in \eqref{J} with $\lim_{t\downarrow 0} \bm\mu_t=\mu_0$. While the general existence of such a gradient flow is known from Proposition~\ref{prop:GF exists}, we show here that it can be explicitly constructed as a time-changed constant-speed geodesic. 

\begin{theorem}\label{thm:mu is GF}
For any $\mu_0\in\Pc^r_2(\R^d)$, $\bm\mu^*$ defined in \eqref{Def: mu t} satisfies the following:
\begin{itemize}
\item [(i)] $\bm\mu^*\in\Pc^r_2(\R^d)$ for all $t\ge 0$;
\item [(ii)] $\bm\mu^*$ is the unique gradient flow for $J:\Pc_2(\R^d)\to\R$ in \eqref{J} with $\lim_{t\downarrow 0} \bm\mu_t=\mu_0$. 
\end{itemize}
\end{theorem}

\begin{proof} 
From Lemma~\ref{lem:has densities}, we immediately see that $\mu_0\in\Pc^r_2(\R^d)$ implies $\bm\mu^*_t = \bm\beta_{1-e^{-t}}\in \Pc_2^r(\R^d)$ for all $t\ge 0$. By the definitions of $\bm\mu^*$ and $\bm\beta$, 
\begin{eqnarray*}
\lim_{t\downarrow 0}W_2(\bm\mu^*_t,\mu_0) =  \lim_{t\downarrow 0}W_2(\bm\beta_{1-e^{-t}},\bm\beta_0)=\lim_{t\downarrow 0} (1-e^{-t})W_2(\mu_0,\mud) = 0,
\end{eqnarray*}
where the second equality follows from \eqref{Eq: s1 s2 constant speed geodesic formula}. That is, $\lim_{t\downarrow 0} \bm\mu^*_t=\mu_0$. Hence, it remains to show that $\bm\mu^*$ is a gradient flow for $J$ in \eqref{J}. As $\bm\mu^*\in AC^{2}((0, \infty); \mathcal{P}_{2}(\mathbb{R}^{d}))$ (recall Lemma~\ref{lem:CSG}), it suffices to prove that $\bm\mu^*$ satisfies \eqref{max slope}, thanks to Proposition~\ref{prop:max slope = GF}. For any $t>0$, as $\frac{d}{dt} J(\bm\mu^*_t)$ and $ |(\bm\mu^*)'|(t)$ have been derived explicitly in Lemma~\ref{lem:CSG}, it remains to compute $|\partial J|(\bm\mu^*_t)$. By \cite[Theorem 10.4.12]{Ambrosio08}, 
\begin{eqnarray*}
|\partial J|(\bm\mu^*_t) = \min\bigg\{\int_{\R^d} |\bar\gamma - \mathbf{i}|^2 d\bm\mu^*_t: \gamma\in\Gamma_0(\bm\mu^*_t,\mud)\bigg\}^{1/2}, 
\end{eqnarray*}
where $\bar\gamma:\R^d\to\R^d$ is the barycentric projection of $\gamma\in\Gamma_0(\bm\mu^*_t,\mud)$, defined by 
\begin{eqnarray*}
\bar\gamma(x) := \int_{\R^d} y d\gamma(dy\mid x)\quad \hbox{for $\bm\mu^*_t$-a.e.\ $x\in\R^d$},  
\end{eqnarray*}
with $\gamma(dy\mid x)$ denoting the conditional probability of $\gamma(\cdot,\cdot)\in\Pc(\R^d\times\R^d)$ given that the first entry is known to be some $x\in\R^d$. As $\bm\mu^*_t\in\Pc_2^r(\R^d)$, $\Gamma_0(\bm\mu^*_t,\mud)$ is a singleton containing only $\gamma=(\mathbf{i}\times\mathbf{t}_{\bm\mu^*_t}^{\mud})_\#\bm\mu^*_t$ (Remark~\ref{rem:mu is regular}). It follows that $\gamma(dy\mid x)$ is the dirac measure concentrated on the point $\mathbf{t}_{\bm\mu^*_t}^{\mud}(x)$, 
such that $\bar\gamma(x) = \mathbf{t}_{\bm\mu^*_t}^{\mud}(x)$, for $\bm\mu^*_t$-a.e.\ $x\in\R^d$. This in turn implies that 
\begin{eqnarray}\label{LS J}
|\partial J|(\bm\mu^*_t) = \left(\int_{\R^d} |\mathbf{t}_{\bm\mu^*_t}^{\mud} - \mathbf{i}|^2 d\bm\mu^*_t\right)^{1/2} = W_2(\bm\mu^*_t,\mud)=e^{-t}W_2(\mu_0,\mud), 
\end{eqnarray}
where the last equality follows from \eqref{Def: mu t}, $\bm\beta_1=\mud$, and \eqref{Eq: s1 s2 constant speed geodesic formula}. By Lemma~\ref{lem:CSG} and \eqref{LS J}, 
\begin{eqnarray*}
\frac{d}{dt} J(\bm\mu^*_t) &= -e^{-2t}W_2^2(\mu_0,\mud) = - (e^{-t} W_2(\mu_0,\mud))^2 =-\frac{1}{2}\left(|\partial J|^2(\bm\mu^*_t) + |(\bm\mu^*)'|^2(t)\right). 
\end{eqnarray*}
That is, \eqref{max slope} is fulfilled, as desired. 
\end{proof}

Now, we are ready to show that a solution to the Fokker--Planck equation \eqref{FP} exists and it is unique up to time-marginal laws. 

\begin{corollary}\label{coro:sol. to FP}
For any $\mu_0\in\Pc_2^r(\R^d)$, $\bm\mu^*$ defined in \eqref{Def: mu t} 
is a solution to the Fokker--Planck equation \eqref{FP} (where $v$ is the velocity field of $\bm\mu^*$) that satisfies
\begin{eqnarray}\label{v conditions'}
\nabla\phi_{\bm\mu^*_t}^{\mud} \in L^2(\R^d,\bm\mu^*_t)\quad\hbox{and}\quad \|\nabla\phi_{\bm\mu^*_t}^{\mud} \|_{L^2(\R^d,\bm\mu^*_t)} = |(\bm\mu^*)'|(t)\quad \hbox{for}\ \mathcal L^1\hbox{-a.e.\ $t > 0$}. 
\end{eqnarray}
Let $\bm\mu\in AC^2_{\operatorname{loc}}((0,\infty),\Pc_2(R^d))$, with $\bm\mu_t\in\Pc^r_2(\R^d)$ for all $t>0$, be a solution to the Fokker--Planck equation \eqref{FP} for some Borel $v:(0,\infty)\times\R^d\to\R^d$. If $v$ satisfies \eqref{v conditions} (or equivalently, \eqref{v conditions'} holds with $\bm\mu$ in place of $\bm\mu^*$), 
then $W_2(\bm\mu_t,\bm\mu^*_t)=0$ for all $t\ge 0$.  
\end{corollary}

\begin{proof}
By Theorem~\ref{thm:mu is GF} (ii), $\bm\mu^*$ is the unique gradient flow for $J$ with $\lim_{t\downarrow 0}\bm\mu^*_t = \mu_0$. In view of Definitions~\ref{def:GF} and \ref{def:velocity}, this implies that $\bm\mu^*$ is a solution to the Fokker--Planck equation \eqref{FP'}, where $v:(0,\infty)\times\R^d\to\R^d$ therein is the velocity field of $\bm\mu^*$ and it satisfies \eqref{v conditions} and $v_t\in-\partial J(\bm\mu^*_t)$ for all $t>0$. As $\bm\mu^*_t\in\Pc_2^r(\R^d)$ for all $t>0$ (by Theorem~\ref{thm:mu is GF} (i)), we know from Lemma~\ref{lem:subdifferential J} that $\partial J(\bm\mu^*_t)$ is a singleton containing only $\nabla\phi_{\bm\mu^*_t}^{\mud}$, whence  $v_t = -\nabla\phi_{\bm\mu^*_t}^{\mud}$ $\bm\mu^*$-a.e., 
for all $t>0$. It follows that $\bm\mu^*$ fulfills \eqref{FP} and satisfies \eqref{v conditions'}. 
Conversely, given $\bm\mu\in AC^2_{\operatorname{loc}}((0,\infty),\Pc_2(R^d))$ with $\bm\mu_t\in\Pc^r_2(\R^d)$ for all $t>0$, we can argue as above that $\partial J(\bm\mu_t) = \{ \nabla\phi_{\bm\mu_t}^{\mud} \}$ for all $t>0$. If $\bm\mu$ fulfills \eqref{FP}, where $v$ is its velocity field, then \eqref{FP'} holds and $v_t = -\nabla\phi_{\bm\mu_t}^{\mud}\in - \partial J(\bm\mu_t)$ for all $t>0$. This, along with \eqref{v conditions}, implies $\bm\mu$ is a gradient flow for $J$ in \eqref{J} with $\lim_{t\downarrow 0} \bm\mu_t=\mu_0$ (recall Definitions~\ref{def:GF} and \ref{def:velocity}). The uniqueness part of Proposition~\ref{prop:GF exists} then yields $W_2(\bm\mu_t,\bm\mu^*_t)=0$ for all $t\ge 0$.
\end{proof}

\begin{remark}
By Benamou-Brenier's theorem (see e.g., \cite[Theorem 5.53]{CD-book-18-I}), for any $\mu_0\in\Pc^r_2(\R^d)$, 
\begin{equation} \label{BB}
    W_{2}^{2}(\mu_0,\mud) = \inf_{\bm{\nu}, {b}} 
    \int_{0}^{1} \int_{\mathbb{R}^{d}} |b(s,x)|^{2} \, d\bm{\nu}_{s}(x) \, ds,  
\end{equation}
where the infimum is taken over pairs $(\bm\nu,b)$, for $\bm\nu\in C([0,1];\Pc_2(\R^d))$ and Borel $b:(0,1)\times\R^d\to\R^d$, satisfying $\partial_s \bm\nu_s  + \Div (b_s \bm\nu_{s}) = 0$ on $(0,1)\times\R^d$ with $\bm\nu_0=\mu_0$ and $\bm\nu_1=\mud$. It is also known from \cite[Remark 5.60]{CD-book-18-I} that the optimal $\bm\nu^*$ in \eqref{BB} is simply the constant-speed geodesic $\bm\beta$ in \eqref{Def: Generalized Geodesic between mu0 and mud}. Note that the gradient flow $\bm\mu^*$ for $J(\cdot) := \frac12 W_2^2(\cdot, \mud)$ is a time-changed constant-speed geodesic (i.e., $\bm\mu^*_{t} = \bm\beta_{1-e^{-t}}$, $t\ge 0$) and it solves the Fokker--Planck equation \eqref{FP}; recall \eqref{Def: mu t} and Corollary~\ref{coro:sol. to FP}. This indicates that our gradient flow $\bm\mu^*$ can recover the optimal solution $(\bm\nu^*,b^*)$ in \eqref{BB} and further provide a concrete characterization. Indeed, under the time change $t(s):=\ln (\frac{1}{1-s})$ for $s\in [0,1]$, the gradient flow becomes $\bm\mu^*_{t(s)}=\bm\beta_s=\bm\nu^*_s$ and \eqref{FP} turns into $\partial_s \bm\nu^*_s  - \Div (({\nabla \phi_{\bm\nu^*_{s}}^{\mud}}/({1-s}) ) \bm\nu^*_{s}) = 0$ on $(0,1)\times\R^d$, which implies $b^*(s,x) = {\nabla \phi_{\bm{\nu}^*_{s}}^{\mud}}(x)/(s-1) = {\nabla \phi_{\bm{\beta}_{s}}^{\mud}}(x)/(s-1)$. Compared with the traditional formula of $b^*$ (see the discussion below \cite[Remark 5.55]{CD-book-18-I}), our formula is more transparent in that it no longer involves the convex conjugate of $\frac12|x|^2-\phi_{\mu_0}^{\mud}(x)$. 
\end{remark}


\section{Solving the Gradient-descent ODE \eqref{new Y}}\label{sec:sol. to ODE}
In this section, relying on the (unique) solution $\bm\mu^*$ to the Fokker--Planck equation \eqref{FP} (recall Corollary~\ref{coro:sol. to FP}), we will first construct a (unique) solution $Y$ to ODE \eqref{new Y} on the strength of Trevisan's superposition principle \cite{Trevisan16}. Next, to facilitate numerical simulation of ODE \eqref{new Y}, we propose an Euler scheme 
and show that it does converge to the laws of $\{Y_t\}_{t\ge 0}$ (i.e., $\bm\mu^*$) in the limit. 

To state precisely what forms 
a solution to ODE \eqref{new Y}, consider the space of continuous paths 
\begin{eqnarray*}
(\Omega,\mathcal F) := (C([0,\infty);\R^d), \mathcal B(C([0,\infty);\R^d)).
\end{eqnarray*}

\begin{definition}\label{def:sol. to ODE}
The process $Y_t(\omega) := \omega(t)$, for all $(t,\omega)\in[0,\infty)\times\Omega$, is said to be a solution to \eqref{new Y} if there exist a Borel $v:(0,\infty)\times\R^d\to\R^d$ and a probability measure $\P$ on $(\Omega,\mathcal F)$ such that
\begin{itemize}
\item [(i)] $\bm\mu_t := \mu^{Y_t}$, the law of $Y_t:\Omega\to\R^d$, belongs to $\Pc_2^r(\R^d)$ for all $t\ge 0$ with $\bm\mu_0=\mu_0$;
\item [(ii)] $v(t,\cdot) = -\nabla \phi^{\mud}_{\bm\mu_t}(\cdot)$ $\bm\mu_t$-a.e.\ for all $t>0$;
\item [(iii)] the collection of paths 
$
\Big\{\omega\in\Omega: \omega(t) = \omega(0) + \int_0^t v(s,\omega(s))\, ds,\ t\ge 0\Big\}
$
has probability one. 
\end{itemize}
\end{definition}

\begin{proposition}\label{thm:sol. to ODE exists}
For any $\mu_0\in\Pc_2^r(\R^d)$, there exists a solution $Y$ to ODE \eqref{new Y} such that $\mu^{Y_t} = \bm\mu^*_t$ for all $t\ge 0$, with $\bm\mu^*$ defined as in \eqref{Def: mu t}. 
\end{proposition}

\begin{proof}
By Lemma~\ref{lem:CSG}, $\bm\mu^*\in AC^2((0,\infty);\Pc_2(\R^d))$. This immediately implies that $t\mapsto \bm\mu^*_t$ is continuous under the $W_2$ distance, and thus narrowly continuous. It also ensures, thanks to Lemma~\ref{lem:velocity exists}, the existence of a unique Borel vector field $v:(0,\infty)\times\R^d\to\R^d$ that fulfills \eqref{v conditions} and \eqref{FP'} (with $\bm\mu$ therein replaced by $\bm\mu^*$), i.e., $v$ is the velocity field of $\bm\mu^*$ (recall Definition~\ref{def:velocity}). Hence, to apply the ``superposition principle'' (i.e., \cite[Theorem 2.5]{Trevisan16}) to construct a solution to $dY_t = v(t, Y_t) dt$, it remains to verify equation (2.3) in \cite{Trevisan16}, which in our case takes the form
\begin{eqnarray}\label{want this}
\int_0^T\int_{\R^d} |v(t,y)| d\bm\mu^*_t(y) dt <\infty\quad\forall T>0, 
\end{eqnarray}
For any $T>0$, note that
\begin{eqnarray*}
\int_0^T\int_{\R^d} |v(t,y)| d\bm\mu^*_t(y) dt = \int_0^T \|v_t\|_{L^1(\R^d,\bm\mu^*_t)}  dt \le \int_0^T \|v_t\|_{L^2(\R^d,\bm\mu^*_t)}  dt = \int_0^T |(\bm\mu^*)'|(t) dt,
\end{eqnarray*}
where the last equality stems from \eqref{v conditions} (with $\bm\mu$ replaced by $\bm\mu^*$). As $\bm\mu^*\in AC((0,\infty);\Pc_2(\R^d))$, Remark~\ref{rem:mu'} indicates $|(\bm\mu^*)'|\in L^1((0,\infty))$. This, along with the previous inequality, implies that \eqref{want this} holds.   
Now, consider $Y_t(\omega):=\omega(t)$ for all $(t,\omega)\in[0,\infty)\times\Omega$. By \cite[Theorem 2.5]{Trevisan16}, there is a probability $\P$ on $(\Omega,\mathcal F)$ such that (i) $\P\circ(Y_t)^{-1} = \bm\mu^*_t$ for all $t\ge 0$, and (ii) for any $f\in C^{1,2}((0,\infty)\times\R^d)$, 
\begin{eqnarray}\label{local martingale}
t\mapsto f(t,\omega(t)) - f(0,\omega(0)) - \int_0^t\left(\partial_s f + v\cdot\nabla f \right)(s,\omega(s)) ds,\quad t\ge0
\end{eqnarray}
is a local martingale under $\P$. As $\bm\mu^*_t\in\Pc_2^r(\R^d)$ for all $t\ge 0$ (by Theorem~\ref{thm:mu is GF}), (i) above already implies that Definition~\ref{def:sol. to ODE} (i) is satisfied. By Corollary~\ref{coro:sol. to FP}, $\bm\mu^*$ fulfills \eqref{FP}, which particularly implies Definition~\ref{def:sol. to ODE} (ii). For each $i\in\{1,2,...,d\}$, by taking $f(t,y) = y_i$ for $y=(y_1,y_2,...,y_d)\in\R^d$ in \eqref{local martingale}, we see that
$
t\mapsto M^{(i)}_t(\omega) := (\omega(t))_i-(\omega(0))_i -\int_0^t (v(s,\omega(s)))_i\, ds
$, $t\ge 0$, 
is a local martingale under $\P$. By the same argument on p.\ 315 of \cite{Karatzas91}, we find that the quadratic variation of $M^{(i)}$ is constantly zero, i.e., $\langle M^{(i)}\rangle_t = 0$ for all $t\ge 0$. As $M^{(i)}$ is a local martingale with zero quadratic variation under $\P$, $M^{(i)}_t =0$ for all $t\ge 0$ $\P$-a.s. That is,  
$(\omega(t))_i = (\omega(0))_i +\int_0^t (v(s,\omega(s)))_i\, ds$, $t\ge 0$, for $\P$-a.e.\ $\omega\in\Omega$. 
Because this holds for all $i\in\{1,2,...,d\}$, we conclude
$\omega(t) = \omega(0)+ \int_0^t  v(s,\omega(s))\, ds$, $t\ge 0$, for $\P$-a.e.\ $\omega\in\Omega$, 
i.e., Definition~\ref{def:sol. to ODE} (iii) holds. Hence, $Y$ is a solution to \eqref{new Y}, according to Definition~\ref{def:sol. to ODE}. 
\end{proof}

In fact, under appropriate regularity conditions, the time-marginal laws of {\it any} solution $Y$ to ODE \eqref{new Y} must coincide with $\bm\mu^*$ in \eqref{Def: mu t}. 

\begin{proposition}\label{prop:unique mu^Y_t}
Given $\mu_0\in\Pc_2^r(\R^d)$, let $Y$ be a solution to \eqref{new Y}. Recall $\bm\mu_t:= \mu^{Y_t}\in\Pc_2^r(\R^d)$, $t\ge 0$, and the Borel $v:(0,\infty)\times\R^d\to\R^d$ from Definition~\ref{def:sol. to ODE}. If $\bm\mu$ belongs to $AC((0,\infty);\Pc_2(\R^d))\cap AC^2_{\operatorname{loc}}((0,\infty);\Pc_2(\R^d))$ and $\|v_t\|_{L^2(\R^d,\bm\mu_t)} \le |\bm\mu'|(t)$ for $\mathcal L^1$-a.e.\ $t>0$, 
then $W_2(\bm\mu_t,\bm\mu^*_t)=0$ for all $t\ge 0$, with $\bm\mu^*$ given by \eqref{Def: mu t}.  
\end{proposition}

\begin{proof}
Let $\P$ be the probability measure on $(\Omega,\mathcal F)$ associated with the solution $Y$ to \eqref{Y}; recall Definition~\ref{def:sol. to ODE}. As $\|v_t\|_{L^2(\R^d,\bm\mu_t)} \le |\bm\mu'|(t)$ for $\mathcal L^1$-a.e.\ $t>0$ 
and $\bm\mu$ belongs to $AC((0,\infty);\Pc_2(\R^d))$, we may follow the same argument below \eqref{want this} to show that
\begin{eqnarray}\label{want this'}
\int_0^T\int_{\R^d} |v(t,y)| d\bm\mu_t(y) dt <\infty\quad\forall T>0. 
\end{eqnarray}
For any $\psi \in C_{c}^{1,2}((0, \infty) \times \mathbb{R}^{d})$, take $T>0$ large such that $\psi(T,y) = \psi(0, y) = 0$ for all $y\in\R^d$. With $Y_t(\omega) = \omega(t)$ for all $(t,\omega)\in[0,\infty)\times\Omega$, we deduce from Definition~\ref{def:sol. to ODE} (iii) that 
$0=\psi(T, {Y}_{T}) - \psi(0, {Y}_{0}) = \int_{0}^{T} \left( \partial_{t} \psi +v\cdot \nabla \psi \right)(t, {Y}_{t}) \, dt$ $\P$-a.s. 
Taking expectation under $\P$ on both sides gives
\begin{eqnarray*}
0 = \E\bigg[\int_{0}^{T} \left( \partial_{t} \psi +v \cdot \nabla \psi \right)(t, {Y}_{t}) \, dt\bigg] = \int_{0}^{T} \int_{\R^d} \left( \partial_{t} \psi + v\cdot \nabla \psi \right)(t,y) d\bm\mu_t(y)\, dt,
\end{eqnarray*}
where the second equality follows from Fubini's theorem, which is applicable here thanks to \eqref{want this'}. As $T\to\infty$, we see that $\bm\mu$ fulfills the Fokker--Planck equation \eqref{FP'}. With \eqref{v conditions} and \eqref{FP'} satisfied, the uniqueness part of Lemma~\ref{lem:velocity exists} implies that $v$ is the velocity field of $\bm\mu$ (recall Definition~\ref{def:velocity}). Also, by \eqref{FP'} and Definition~\ref{def:sol. to ODE} (ii), $\bm\mu$ fulfills \eqref{FP}. Now, given that $\bm\mu$ belongs to $AC^2_{\operatorname{loc}}((0,\infty);\Pc_2(\R^d))$ and satisfies \eqref{FP} and its velocity filed $v$ fulfills \eqref{v conditions},  
we conclude from Corollary~\ref{coro:sol. to FP} that $W_2(\bm\mu_t,\bm\mu^*_t)=0$ for all $t\ge 0$. 
\end{proof}

Now, we are ready to present the main theoretic result of this paper. 

\begin{theorem}\label{thm:main}
For any $\mu_0\in\Pc_2^r(\R^d)$, there exists a solution $Y$ to ODE \eqref{new Y} such that $\bm\mu_t:= \mu^{Y_t}$ belongs to $\Pc_2^r(\R^d)$ for all $t\ge 0$ and satisfies
\begin{align}
&\bm\mu\in AC((0,\infty);\Pc_2(\R^d))\cap AC^2_{\operatorname{loc}}((0,\infty);\Pc_2(\R^d))\label{A},\\
&\|\nabla\phi_{\bm\mu_t}^{\mud} \|_{L^2(\R^d,\bm\mu_t)} \le |\bm\mu'|(t)\quad \hbox{for}\ \mathcal L^1\hbox{-a.e.\ $t > 0$}. \label{B}
\end{align}
Moreover, for any such a solution $Y$ to ODE \eqref{new Y}, $W_2(\mu^{Y_t},\mud)\downarrow 0$ as $t\to\infty$; specifically, 
\begin{eqnarray*}
W_2(\mu^{Y_t},\mud) \le e^{-t} W_2(\mu_0,\mud)\quad \forall t\ge 0. 
\end{eqnarray*}
\end{theorem}

\begin{proof}
Recall from Lemma~\ref{lem:CSG}, and Corollary~\ref{coro:sol. to FP} that $\bm\mu^*$ defined in \eqref{Def: mu t} fulfills \eqref{A} and \eqref{B}. Hence, the solution $Y$ to ODE \eqref{new Y} constructed in Proposition~\ref{thm:sol. to ODE exists} readily satisfies all the conditions in the theorem. As $\bm\mu^*$ is the unique gradient flow for $J$ in \eqref{J} with $\lim_{t\downarrow 0} \bm\mu^*_t = \mu_0$ (by Theorem~\ref{thm:mu is GF}), we know from Proposition~\ref{prop:GF exists} that 
\begin{eqnarray}\label{C}
J(\bm\mu^*_t)\downarrow 0\ \hbox{as}\ t\to\infty\quad \hbox{and}\quad J(\bm\mu^*_t)\le e^{-2t} J(\mu_0)\ \ \forall t\ge 0. 
\end{eqnarray}
Let $Y$ be an arbitrary solution to ODE \eqref{new Y} that fulfills all the conditions in the theorem. Recall the Borel $v:(0,\infty)\times\R^d\to\R^d$ from Definition~\ref{def:sol. to ODE}. Note that Definition~\ref{def:sol. to ODE} (ii) and \eqref{B} imply $\|v_t\|_{L^2(\R^d,\bm\mu_t)} \le |\bm\mu'|(t)$ for $\mathcal L^1$-a.e.\ $t>0$. Proposition~\ref{prop:unique mu^Y_t} then asserts that $\mu^{Y_t} = \bm\mu^*_t$ for all $t\ge 0$. By \eqref{J}, this implies $W_2(\mu^{Y_t},\mud) = W_2(\bm\mu^*_t,\mud) = (2 J(\bm\mu^*_t))^{1/2}$ for all $t\ge 0$. Thus, \eqref{C} implies $W_2(\mu^{Y_t},\mud)\downarrow 0$ and $W_2(\mu^{Y_t},\mud) \le e^{-t} W_2(\mu_0,\mud)$ for all $t\ge 0$.
\end{proof}

Theorem~\ref{thm:main} suggests that one can uncover the true data distribution $\mud\in\Pc_2(\R^d)$ by simulating the distribution-dependent ODE \eqref{new Y}. To facilitate the numerical implementation of this, we now propose a discretization of \eqref{new Y} and closely study its convergence as the time step diminishes.


\subsection{An Euler Scheme}\label{subsec:Euler}
As ODE \eqref{new Y} is distribution-dependent, when we discretize it, we need to simulate a random variable and keep track of its law at every discrete time point. Specifically, fix a time step $0<\eps <1$. At time 0, we simulate a random variable $Y^\eps_0$ according to a given initial distribution, i.e., $\mu^{Y^\eps_0}=\mu_{0,\eps} := \mu_0 \in\Pc_2^r(\R^d)$. At time $\eps$, we simulate the random variable 
\begin{eqnarray*}
Y^\eps_1 := Y^\eps_0 -\eps \nabla\phi_{\mu_0}^{\mud}(Y^\eps_0),\quad \hbox{such that}\quad \mu^{Y^\eps_1} =  \mu_{1, \eps} := (\mathbf{i} - \eps\nabla\phi_{\mu_0}^{\mud})_{\#} \mu_{0},
\end{eqnarray*} 
where $\phi_{\mu_0}^{\mud}$ is a Kantorovich potential from $\mu_0$ to $\mud$. It is crucial that $\mu_0$ lies in $\Pc_2^r(\R^d)$, instead of the more general $\Pc_2(\R^d)$. Indeed, as $\nabla\phi_{\mu_0}^{\mud}$ is only well-defined $\mathcal L^d$-a.e.\ (Remark~\ref{rem:nabla u}), $\mu_0\ll\mathcal L^d$ ensures that $Y^\eps_1$ above is well-defined. Moreover, with $\mu_0\in\Pc_2^r(\R^d)$, we may apply the same arguments below \eqref{gamma formula} and find that $\mu_{1,\eps}$ remains in $\Pc_2^r(\R^d)$. At time $2\eps$, we simulate the random variable 
\begin{eqnarray*}
Y^\eps_2 := Y^\eps_1 -\eps \nabla\phi_{\mu_{1,\eps}}^{\mud}(Y^\eps_1),\quad \hbox{such that}\quad \mu^{Y^\eps_2} =  \mu_{2, \eps} := (\mathbf{i} - \eps\nabla\phi_{\mu_{1,\eps}}^{\mud})_{\#} \mu_{1,\eps},
\end{eqnarray*}
where $\phi_{\mu_{1,\eps}}^{\mud}$ is a Kantorovich potential from $\mu_{1,\eps}$ to $\mud$. Again, $\mu_{1,\eps}\in\Pc_2^r(\R^d)$ ensures that $Y^\eps_2$ above is well-defined and allows the use of the arguments below \eqref{gamma formula}, which in turn guarantees $\mu_{2,\eps}\in\Pc_2^r(\R^d)$. Continuing this procedure, we simulate for all $n\in\N$ the random variables
\begin{eqnarray}\label{Euler}
Y^\eps_{n} := Y^\eps_{n-1} -\eps \nabla\phi_{\mu_{n-1,\eps}}^{\mud}(Y^\eps_{n-1}),\quad \hbox{such that}\quad \mu^{Y^\eps_{n}} =  \mu_{n, \eps} := (\mathbf{i} - \eps\nabla\phi_{\mu_{n-1,\eps}}^{\mud})_{\#} \mu_{n-1,\eps},
\end{eqnarray}
and we know $\mu_{n, \eps}\in\Pc^r_2(\R^d)$ for all $n\in\N$. This allows us to invoke Remark~\ref{rem:mu is regular}, particularly \eqref{t and u}, to express $\{\mu_{n,\eps}\}_{n\in\N}$ equivalently as  
\begin{eqnarray}\label{mu_n}
\mu_{n, \eps} = \left(\mathbf{i} + \eps (\mathbf{t}_{\mu_{n-1, \eps}}^{\mud} - \mathbf{i})\right)_{\#} \mu_{n-1, \eps}\quad\forall n\in\N,
\end{eqnarray}
using their associated optimal transport maps. Interestingly, how $\mu_{n-1,\eps}$ is transported to $\mu_{n,\eps}$ in \eqref{mu_n} is already optimal, as the next result shows. 

\begin{lemma}\label{lem:mu_n is optimal}
Fix $0<\eps<1$. For any $n\in\N$, $\mathbf{i} + \eps (\mathbf{t}_{\mu_{n-1, \eps}}^{\mud} - \mathbf{i}):\R^d\to\R^d$ is an optimal transport map from $\mu_{n-1, \eps}$ to $\mu_{n, \eps}$. That is, $\mathbf t_{\mu_{n-1, \eps}}^{\mu_{n, \eps}} = \mathbf{i} + \eps (\mathbf{t}_{\mu_{n-1, \eps}}^{\mud} - \mathbf{i})$ for all $n\in\N$. 
\end{lemma}

\begin{proof}
For any $n\in\N$, by \cite[Theorem 7.2.2]{Ambrosio08}, $\bm\nu_{s}:= ((1-s)\mathbf{i} +s \mathbf{t}_{\mu_{n-1, \eps}}^{\mud})_{\#} \mu_{n-1, \eps}$, $s\in[0,1]$, is a constant-speed geodesic connecting $\mu_{n-1,\eps}$ to $\mud$, i.e., it satisfies, per \cite[Definition 2.4.2]{Ambrosio08}, 
\begin{eqnarray} \label{Eq: s1 s2 constant speed geodesic formula'}
   W_{2}(\bm\nu_{s_{1}}, \bm\nu_{s_{2}}) = (s_{2} - s_{1}) W_{2}(\mu_{n-1,\eps}, \mud), \quad \forall 0 \leq s_{1} \leq s_{2} \leq 1.
\end{eqnarray}
Observe from \eqref{mu_n} that $\mu_{n,\eps} = \bm\nu_\eps$. Hence, 
\begin{eqnarray}\label{W_2=epsW_2}
W_2(\mu_{n-1,\eps},\mu_{n,\eps}) = W_2(\bm\nu_0,\bm\nu_\eps) = \eps W_2(\mu_{n-1,\eps},\mud),
\end{eqnarray}
where the last equality follows from \eqref{Eq: s1 s2 constant speed geodesic formula'}. Now, set $\mathbf t:= \mathbf{i} + \eps (\mathbf{t}_{\mu_{n-1, \eps}}^{\mud} - \mathbf{i})$ and note that
\begin{eqnarray*}
\int_{\R^d} |\mathbf t -\mathbf i|^2 d\mu_{n-1,\eps} = \eps^2 \int_{\R^d} |\mathbf{t}_{\mu_{n-1, \eps}}^{\mud}-\mathbf i|^2 d\mu_{n-1,\eps}=\eps^2 W^2_2(\mu_{n-1,\eps},\mud) = W^2_2(\mu_{n-1,\eps},\mu_{n,\eps}),
\end{eqnarray*}
where the second equality follows from the fact that $\mathbf{t}_{\mu_{n-1, \eps}}^{\mud}$ is an optimal transport map from $\mu_{n-1,\eps}$ to $\mud$ and the third equality stems from \eqref{W_2=epsW_2}. This, together with \eqref{mu_n}, readily shows that $\mathbf t:= \mathbf{i} + \eps (\mathbf{t}_{\mu_{n-1, \eps}}^{\mud} - \mathbf{i})$ is an optimal transport map from $\mu_{n-1,\eps}$ to $\mu_{n,\eps}$.   
\end{proof}

An important consequence of Lemma~\ref{lem:mu_n is optimal} is that the measures $\{\mu_{n,\eps}\}_{n\in\N}$ {\it all} lie on the constant-speed geodesic connecting $\mu_0$ and $\mud$, i.e., $\bm\beta$ given in \eqref{beta'}, or equivalently, on the curve $\bm\mu^*$ defined in \eqref{Def: mu t}. To see this, we consider, for each $0<\eps<1$, the unique $t_\eps>0$ such that
\begin{eqnarray}\label{t_eps}
\eps = 1 - e^{-t_{\eps}}. 
\end{eqnarray}

\begin{remark}\label{rem:t_eps}
It can be checked directly from \eqref{t_eps} that (i) $t_\eps>\eps$ for any $0<\eps<1$; (ii) as $\eps\downarrow 0$, we have $t_\eps\downarrow0$ and $\frac{t_\eps}{\eps}\to 1$.  
\end{remark}

\begin{proposition}\label{prop:mu_n on beta}
For any $0<\eps<1$, we have
$\mu_{n,\eps} = \bm\beta_{1-e^{-n t_\eps}} = \bm\mu^*_{n t_\eps}$  for all $n\in\N$,
where $t_\eps>0$ is given by \eqref{t_eps}. 
\end{proposition}

\begin{proof}
We will prove by induction that 
\begin{eqnarray}\label{we claim}
\mu_{n,\eps} = \bm\beta_{1-e^{-n t_\eps}} = \bm\mu^*_{n t_\eps}\quad\hbox{and}\quad \mathbf{t}_{\mu_{0}}^{\mu_{n, \eps}} = \mathbf{i} + (1-e^{-nt_{\eps}})(\mathbf{t}_{\mu_{0}}^{\mud} - \mathbf{i})\qquad \forall n\in\N,
\end{eqnarray}
By \eqref{mu_n} with $n=1$, $\mu_{0,\eps}=\mu_0$, and the definitions of $\bm\beta$ and $\bm\mu^*$, we see that $\mu_{1, \eps} = \bm\beta_{\eps} = \bm\beta_{1-e^{-t_{\eps}}} = \bm\mu^*_{t_\eps}$. Also, by Lemma~\ref{lem:mu_n is optimal} and $\mu_{0,\eps}=\mu_0$, we have $\mathbf t_{\mu_{0}}^{\mu_{1, \eps}} = \mathbf{i} + \eps (\mathbf{t}_{\mu_{0}}^{\mud} - \mathbf{i})=\mathbf{i} + (1-e^{-t_{\eps}})(\mathbf{t}_{\mu_{0}}^{\mud} - \mathbf{i})$. That is, \eqref{we claim} holds for $n=1$. Now, suppose that \eqref{we claim} holds for some $n=k-1\in \N$, i.e., 
\begin{eqnarray}\label{induction}
\mu_{k-1,\eps} = \bm\beta_{1-e^{-(k-1)t_\eps}} = \bm\mu^*_{(k-1) t_\eps}\quad\hbox{and}\quad \mathbf{t}_{\mu_{0}}^{\mu_{k-1, \eps}} = \mathbf{i} + (1-e^{-(k-1)t_{\eps}})(\mathbf{t}_{\mu_{0}}^{\mud} - \mathbf{i}), 
\end{eqnarray}
and we aim to show that \eqref{we claim} remains true for $n=k$. By \eqref{mu_n}, $\mu_{k-1,\eps}={\mathbf t_{\mu_0}^{\mu_{k-1,\eps}}}_\#\mu_0$, and the composition rule for pushforward measures (see e.g., \cite[(5.2.4)]{Ambrosio08}), 
\begin{eqnarray*}
        \begin{aligned}
            \mu_{k, \eps} &= \big(\mathbf{i} + \eps (\mathbf{t}_{\mu_{k-1,\eps}}^{\mud} - \mathbf{i})\big)_{\#}\mu_{k-1, \eps} = \big(\mathbf{i} + \eps (\mathbf{t}_{\mu_{k-1,\eps}}^{\mud} - \mathbf{i})\big)_{\#} ({\mathbf{t}_{0}^{\mu_{k-1, \eps}}}_{\#} \mu_{0}) \\
            &= \Big(\mathbf{t}_{\mu_0}^{\mu_{k-1, \eps}} + \eps \big(\mathbf{t}_{\mu_{k-1,\eps}}^{\mud} \circ \mathbf{t}_{\mu_0}^{\mu_{k-1, \eps}} - \mathbf{t}_{\mu_0}^{\mu_{k-1, \eps}}\big)\Big)_{\#} \mu_{0}.         
            \end{aligned}
\end{eqnarray*}
As $\mu_{k-1,\eps}$ lies on the constant-speed geodesic $\bm\beta$ from $\mu_0$ to $\mud$ (by the first relation in \eqref{induction}), \cite[Lemma 7.2.1]{Ambrosio08} implies that $\mathbf t_{\mu_0}^{\mud}=\mathbf{t}_{\mu_{k-1,\eps}}^{\mud} \circ \mathbf{t}_{\mu_0}^{\mu_{k-1, \eps}}$. Plugging the second relation in \eqref{induction}, $\eps = 1 - e^{-t_{\eps}}$, and $\mathbf t_{\mu_0}^{\mud}=\mathbf{t}_{\mu_{k-1,\eps}}^{\mud} \circ \mathbf{t}_{\mu_0}^{\mu_{k-1, \eps}}$ into the previous equality, we get
\begin{eqnarray}
\begin{aligned}
 \mu_{k, \eps} &= \bigg(\mathbf{i} + (1-e^{-(k-1)t_{\eps}})(\mathbf{t}_{\mu_0}^{\mud} - \mathbf{i}) \notag\\
 &\hspace{0.6in}+ (1 - e^{-t_{\eps}})\left(\mathbf{t}_{\mu_0}^{\mud} - \Big(\mathbf{i} + (1-e^{-(k-1)t_{\eps}})(\mathbf{t}_{\mu_0}^{\mud} - \mathbf{i})\Big)\right)\bigg)_{\#} \mu_{0} \notag\\
                        &= \Big(e^{-kt_{\eps}} \mathbf{i} + (1-e^{-kt_{\eps}})\mathbf{t}_{\mu_0}^{\mud}\Big)_{\#}\mu_{0}= \bm\mu^*_{kt_{\eps}}=\bm\beta_{1-e^{-k t_\eps}},\label{we claim 1}
\end{aligned}
\end{eqnarray}    
where the second last equality is due to \eqref{mu^*'}. Now, we see that $\mu_0$, $\mu_{k-1,\eps}$, and $\mu_{k,\eps}$ are all on the the constant-speed geodesic $\bm\beta$ connecting $\mu_0$ and $\mud$, simply at the different time points 0, $1-e^{-(k-1) t_\eps}$, and $1-e^{-k t_\eps}$, respectively. Hence, by a proper change of time, there exists a constant-speed geodesic that connects $\mu_0$ and $\mu_{k,\eps}$ and passes through $\mu_{k-1,\eps}$. By applying \cite[Lemma 7.2.1]{Ambrosio08} to this constant-speed geodesic, we have 
\begin{eqnarray}
\begin{aligned}
\mathbf t_{\mu_0}^{\mu_{k,\eps}} &= \mathbf t_{\mu_{k-1,\eps}}^{\mu_{k,\eps}}\circ \mathbf t_{\mu_0}^{\mu_{k-1,\eps}} =\mathbf t_{\mu_0}^{\mu_{k-1,\eps}}+ \eps (\mathbf{t}_{\mu_{k-1, \eps}}^{\mud}\circ \mathbf t_{\mu_0}^{\mu_{k-1,\eps}} - \mathbf t_{\mu_0}^{\mu_{k-1,\eps}})\notag\\
&=  \mathbf{i} + (1-e^{-(k-1)t_{\eps}})(\mathbf{t}_{\mu_{0}}^{\mud} - \mathbf{i}) +(1-e^{-t_\eps}) (\mathbf t_{\mu_0}^{\mud} -  \mathbf{i} - (1-e^{-(k-1)t_{\eps}})(\mathbf{t}_{\mu_{0}}^{\mud} - \mathbf{i}))\notag\\
&=e^{-kt_{\eps}} \mathbf{i} + (1-e^{-kt_{\eps}})\mathbf{t}_{\mu_0}^{\mud} = \mathbf{i} + (1-e^{-k t_{\eps}})(\mathbf{t}_{\mu_{0}}^{\mud} - \mathbf{i}),\label{we claim 2}
\end{aligned}
\end{eqnarray}
where the second equality is due to $\mathbf t_{\mu_{k-1, \eps}}^{\mu_{k, \eps}} = \mathbf{i} + \eps (\mathbf{t}_{\mu_{k-1, \eps}}^{\mud} - \mathbf{i})$ (Lemma~\ref{lem:mu_n is optimal}) and the third equality follows from plugging the second relation in \eqref{induction}, $\eps = 1 - e^{-t_{\eps}}$, and $\mathbf t_{\mu_0}^{\mud}=\mathbf{t}_{\mu_{k-1,\eps}}^{\mud} \circ \mathbf{t}_{\mu_0}^{\mu_{k-1, \eps}}$ into the first line. By \eqref{we claim 1} and \eqref{we claim 2}, we see that \eqref{we claim} holds for $n=k$, as desired. 
\end{proof}

Given $0<\eps<1$, we now define a piecewise-constant flow of measures $\bm\mu^\eps:[0,\infty)\to \Pc_2(\R^d)$ by 
\begin{eqnarray} \label{Def: W2 Euler peicewise step}
\bm\mu^{\eps}_t := \mu_{n, \eps}\quad \hbox{if}\quad t\in [n \eps, (n+1) \eps),\qquad \forall n\in\N\cup\{0\}.
\end{eqnarray}
That is, $\bm\mu^\eps$ is the flow of measures obtained under our Euler scheme \eqref{Euler}. As $\eps\downarrow 0$, the next result shows that $\bm\mu^\eps$ converges to $\bm\mu^*$ defined in \eqref{mu^*'}. 

\begin{theorem} \label{Th: Convergence of Euler discretization}
For any $\mu_0\in\Pc_2^r(\R^d)$, 
       $ \lim_{\eps \downarrow 0} W_{2}(\bm\mu^{\eps}_t, \bm\mu^*_{t}) = 0$  for all $t\in[0,\infty)$.
\end{theorem}

\begin{proof}
For any $t\in [0,\infty)$, there exists some $n\in\N\cup\{0\}$ such that $t \in [n\eps, (n+1)\eps)$. Then, 
\begin{eqnarray*}
W_{2}(\bm\mu^{\eps}_t, \bm\mu^*_{t}) = W_{2}(\mu_{n, \eps}, \bm\mu^*_{t}) = W_{2}(\bm\mu^*_{nt_{\eps}}, \bm\mu^*_{t}) =   \bigg(\int_{\min\{nt_{\eps},t\}}^{\max\{nt_{\eps},t\}} e^{-s}ds\bigg)  W_{2}(\mu_{0}, \mud), 
\end{eqnarray*}
where the second equality follows from Proposition~\ref{prop:mu_n on beta} (with $t_\eps>0$ defined as in \eqref{t_eps}) and the last equality results from a calculation similar to \eqref{Eq: mut is AC curve}. As $t_\eps>\eps$ (by Remark~\ref{rem:t_eps}) and $t \in [n\eps, (n+1)\eps)$, we have $n\eps \le \min\{nt_{\eps},t\} \le \max\{nt_{\eps},t\}\le (n+1)t_\eps$. Hence, the previous equality implies
\begin{eqnarray*}
\begin{aligned}
W_{2}(\bm\mu^{\eps}_t, \bm\mu^*_{t}) &\le \bigg(\int_{n \eps}^{(n+1)t_{\eps}} e^{-s}ds\bigg)  W_{2}(\mu_{0}, \mud) = \left(e^{-n \eps} - e^{-(n+1)t_{\eps}}\right) W_{2}(\mu_{0}, \mud)\\
& = e^{-(n+1) t_{\eps}} \left(e^{n(t_{\eps}-\eps)+t_\eps}-1\right) W_{2}(\mu_{0}, \mud) \\
&\le \left(e^{n\eps\big(\frac{t_{\eps}}{\eps}-1\big)+t_\eps}-1\right) W_{2}(\mu_{0}, \mud) \le \left(e^{t\big(\frac{t_{\eps}}{\eps}-1\big)+t_\eps}-1\right) W_{2}(\mu_{0}, \mud),
\end{aligned}
\end{eqnarray*}
where the last inequality follows from $t\ge n\eps$ and $t_\eps>\eps$. As $\eps\downarrow 0$, since $t_\eps\downarrow 0$ and $\frac{t_\eps}{\eps}\to 1$ (by Remark~\ref{rem:t_eps}), the right-hand side above vanishes, which entails $W_{2}(\bm\mu^{\eps}_t, \bm\mu^*_{t})\to 0$.   
\end{proof}

Recall from Proposition~\ref{prop:unique mu^Y_t} that $\bm\mu^*$ is the unique flow of probability measures induced by any solution $Y$ to ODE \eqref{new Y} under suitable regularity. Hence, Theorem~\ref{Th: Convergence of Euler discretization} stipulates that our Euler scheme \eqref{Euler} does converge to the right limit, recovering the time-marginal laws of ODE \eqref{new Y}.



\section{Algorithms}\label{sec:W2-FE}
To simulate the distribution-dependent ODE \eqref{new Y}, we follow the Euler scheme \eqref{Euler} in designing Algorithm~\ref{Alg: Simulate W2 GF ODE} (called {W2-FE}, where ``{FE}'' means ``forward Euler''). The algorithm uses three deep neural networks. One network trains a ``generator'' $G_{\theta} : \mathbb{R}^{d} \rightarrow \mathbb{R}^{d}$ to properly represent the evolving time-marginal laws of ODE \eqref{new Y}. The other two networks, by following \cite[Algorithm 1]{Seguy18}, estimate the Kantorovich potential from the distribution induced by $G_\theta$ to the data distribution $\mud$ (denoted by $\phi_{w}$) as well as its $c$-transform (denoted by $\psi_{v}$). 

  
 
How $G_\theta$ is updated demands a closer examination. Based on the estimated Kantorovich potential $\phi_w$, $G_\theta$ is trained via the Euler scheme \eqref{Euler}. Given a sample $z_i$ from a prior distribution, the point $y_i = G_\theta(z_i)$ represents a sample from $\mu^{Y^\eps_{n}}$. A forward Euler step yields a new point $\zeta_i$, which represents a sample from $\mu^{Y^\eps_{n+1}}$. The generator's task is to produce samples indistinguishable from the new points $\{\zeta_i\}$---or more precisely, to learn the distribution $\mu^{Y^\eps_{n+1}}$, represented by the points $\{\zeta_i\}$. To this end, we fix the points $\{\zeta_i\}$ and update $G_\theta$ by descending the mean square error (MSE) between $\{G_\theta(z_i)\}$ and $\{\zeta_i\}$ up to $K\in \N$ times. Let us stress that throughout the $K$ updates of $G_\theta$, the points $\{\zeta_i\}$ are kept unchanged. This sets us apart from the standard implementation of stochastic gradient descent (SGD), but for a good reason: as our goal is to learn the distribution represented by $\{\zeta_i\}$, it is important to keep $\{\zeta_i\}$ unchanged for the eventual $G_\theta$ to more accurately represent $\mu^{Y^\eps_{n+1}}$, such that the (discretized) ODE \eqref{Euler} is more closely followed.


How we update $G_\theta$ is reminiscent of {\it persistent training} in Fischetti et al.\ \cite{Fischetti18}, a technique consisting of reusing the same minibatch for $K$ consecutive SGD iterations. As shown in \cite[Figure 1]{Fischetti18}, a persistency level of five (i.e., $K = 5$) allows much faster convergence on the CIFAR-10 dataset. Our numerical examples in Section~\ref{sec:examples} also indicate a similar effect of an enlarged persistency level. 

\begin{algorithm}
\caption{W2-FE, our proposed algorithm}
\label{Alg: Simulate W2 GF ODE}
\begin{algorithmic}
\Require Input measures $\mu^{Z}, \mud$, batch sizes $m$, learning rates $\gamma_{g}, \gamma_{d}$, regularizer $\lambda$, time step $\Delta t$, persistency level $K$, Kantorovich potential pair $(\psi_{v}, \phi_{w})$ parameterized as deep neural networks, and generator $G_{\theta}$ parameterized as a deep neural network. 
\For{Number of training epochs}
    \For{Number of updates} \Comment{Kantorovich potential update}
        \State Sample a batch $\{ x_{i} \}$ from $\mud$ and a batch $\{ z_{i} \}$ from $\mu^{Z}$ 
        \State $L_{D} \gets \frac{1}{m} \sum_{i=1}^{m} \left[ \phi_{w}(G_{\theta}(z_{i})) + \psi_{v}(x_{i}) - \lambda \left( \phi_{w}(G_{\theta}(z_{i})) + \psi_{v}(x_{i}) - \frac{1}{2} ||G_{\theta}(z_{i}) - x_{i}||_{2}^{2} \right)_{+}  \right]$ 
        \State $w \gets w + \gamma_{d} \nabla_{w} {L}_{D}$ 
        \State $v \gets v + \gamma_{d} \nabla_{v} {L}_{D}$
    \EndFor
    \State Sample a batch $\{ z_{i} \}$ from $\mu^{Z}$ 
    \State Generate $y_{i} \gets G_{\theta}(z_{i})$
    \State Generate $\zeta_{i} \gets y_{i} - \Delta t \nabla_{} \phi_w(y_{i}) $
    \For{$K$ generator updates} \Comment{Generator update by persistent training}
        \State $\theta \gets \theta - \frac{\gamma_{g}}{m} \nabla_{\theta} \sum_{i=1}^m |\zeta_i - G_{\theta}(z_{i}) |^2$
    \EndFor
\EndFor
\end{algorithmic}
\end{algorithm}


\subsection{Connections to Generative Adversarial Networks (GANs)}
GANs of Goodfellow et al.\ \cite{Goodfellow14} also approach \eqref{to solve}---yet with the $W_2$ distance replaced by the Jensen-Shannon divergence (JSD). The key idea of \cite{Goodfellow14} is to minimize the JSD loss by a min-max game between a generator and a discriminator. While the GAN algorithm had been impactful on artificial intelligence, it was well-known for its instability; see e.g., \cite{Goodfellow16} and \cite{Metz17}. A popular solution is to replace the distance function JSD by the first-order Wasserstein distance ($W_1$). This yields the Wasserstein GAN (WGAN) algorithm \cite{Arjovsky17}, which minimizes the $W_1$ loss by the same min-max game idea in \cite{Goodfellow14}. As shown in \cite[Figure 2]{Arjovsky17}, WGAN performs well where early versions of GANs fare poorly. 

If one further replaces the distance function $W_1$ by $W_2$, the problem becomes exactly \eqref{to solve}. Note that \eqref{to solve} has received less attention in the literature (relative to the $W_1$ counterpart), with only a few exceptions (e.g., \cite{Leygonie19} and \cite{Huang-ISIT-23}). In particular, to solve \eqref{to solve}, \cite{Leygonie19} develops an algorithm (called W2-GAN) by following the same min-max game idea in \cite{Goodfellow14}. While we solve \eqref{to solve} by a different approach (i.e., the gradient-flow perspective), it is interesting to note that our algorithm W2-FE covers W2-GAN as a special case. For $K = 1$ in W2-FE, the generator update reduces to standard SGD without persistent training, which in turn equates W2-FE with  W2-GAN.

\begin{proposition} \label{prop:equivalence}
W2-GAN (i.e., \cite[Algorithm 1]{Leygonie19}) is equivalent to W2-FE (i.e., Algorithm~\ref{Alg: Simulate W2 GF ODE} above) with $K=1$, up to an adjustment of learning rates. 
\end{proposition}
\begin{proof}
As W2-GAN and W2-FE update the Kantorovich potential pair $(\phi_{w}, \psi_{v})$ in the same way, it suffices to show that their generator updates are equivalent under $K=1$. Observe that with $K=1$, 
\begin{eqnarray}\label{equi. calc.}
\begin{split}
    \nabla_{\theta} \frac{1}{m} \sum_{i=1}^{m} |\zeta_{i} - G_{\theta}(z_{i})|^{2} &= -\frac{2}{m} \sum_{i = 1}^{m} (\zeta_{i} - G_{\theta}(z_{i})) \nabla_{\theta}G_{\theta}(z_{i})\\
    & = \frac{2}{m} \sum_{i = 1}^{m} \Delta t \nabla\phi_{w}(G_{\theta}(z_{i})) \nabla_{\theta}G_{\theta}(z_{i}) =  2  \Delta t \nabla_\theta \bigg(\frac{1}{m}  \sum_{i = 1}^{m} \phi_{w}(G_{\theta}(z_{i}))\bigg),
\end{split}
\end{eqnarray}
where the second equality stems from $\zeta_{i} = G_{\theta}(z_{i}) - \Delta t \nabla \phi_{w}(G_{\theta}(z_{i}))$, due to the two lines above the generator update in W2-FE. Thus, the generator update in W2-FE takes the form
    $\theta \gets \theta - \gamma_{g} 2 \Delta t \nabla_{\theta} \frac{1}{m}\sum_{i=1}^{m} \phi_w(G_{\theta}(z_{i}))$,
the same as that in W2-GAN with learning rate $\gamma_{g} 2 \Delta t$. 
\end{proof}

It is tempting to think that if we carry out the generator update $K>1$ consecutive times in W2-GAN, W2-GAN will be equivalent to W2-FE with the same $K>1$. In fact, when $G_\theta$ is updated for the second time in W2-FE, the second equality in \eqref{equi. calc.} no longer holds, as ``$\zeta_{i} = G_{\theta}(z_{i}) - \Delta t \nabla \phi_{w}(G_{\theta}(z_{i}))$'' is true only when $G_\theta$ has not been updated. 
That is, the generator updates in W2-FE and W2-GAN coincide only when $K=1$, and can be quite different for $K>1$.





\section{Numerical Experiments}\label{sec:examples}
First, as in \cite[Section 3.1]{Metz17}, we consider learning a two-dimensional mixture of Gaussians arranged on a circle from an initially given Gaussian distribution. 
We carry out this task using Algorithm~\ref{Alg: Simulate W2 GF ODE} (W2-FE) with varying $K$ values and compare its performance with that of the refined WGAN algorithm in \cite{Petzka18} (called W1-LP), which is arguably one of the most well-performing and stable WGAN algorithms. Figure~\ref{fig: 2D Qualitative Results} shows the qualitative evolution of the models through time. 

Figure \ref{fig: 2D Quantitative Results}, on the other hand, presents the actual $W_1$ and $W_2$ losses through time. We see that W1-LP converges to the same loss level as W2-FE and achieves this faster than W2-FE with $K=1$ (i.e., without persistent training). However, W2-FE with a higher $K$ value converges much faster than W1-LP in both number of epochs and wall-clock time. With $K=10$, W2-FE achieves convergence within $100$ training epochs, while W1-LP needs more than $600$ training epochs. 



\begin{figure}
\centering
\includegraphics[width=15cm]{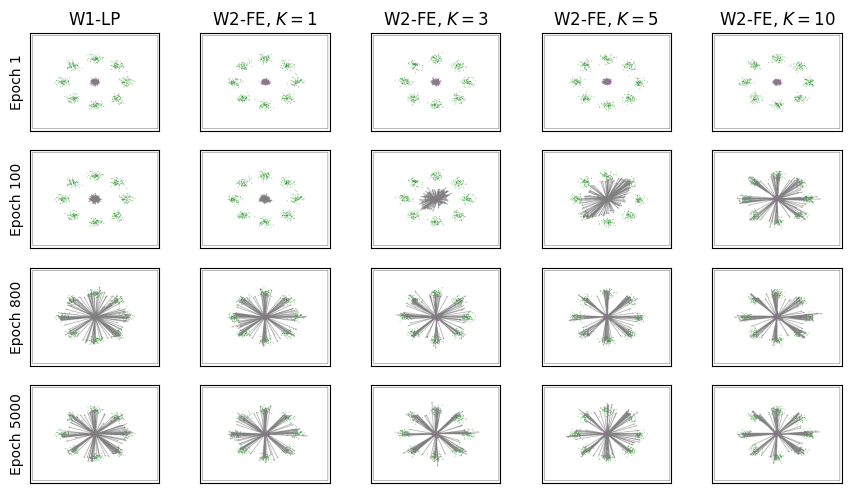}
\caption{\label{fig: 2D Qualitative Results} Qualitative evolution of learning a mixture of Gaussians on a circle (in green) from an initial Gaussian at the center (in magenta). Each grey arrow indicates how a sample from the initial Gaussian is transported toward the target distribution.}
\end{figure}

\begin{figure}
\centering
\includegraphics[width=15.5cm]{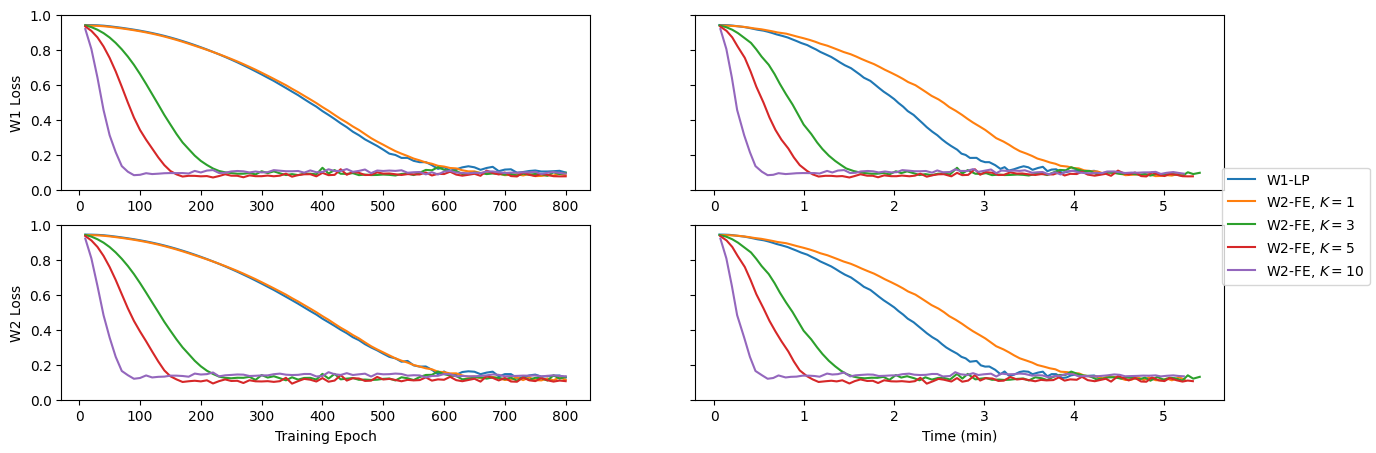}
\caption{\label{fig: 2D Quantitative Results} The first (resp.\ second) row plots the Wasserstein-$1$ (resp.\ Wasserstein-2) loss against training epoch and wall-clock time, respectively, for learning a mixture of Gaussians on a circle from an initial Gaussian at the center.}
\end{figure}

Next, we consider domain adaptation from the MNIST dataset to the USPS dataset. We carry out this task using Algorithm~\ref{Alg: Simulate W2 GF ODE} (W2-FE) with varying $K$ values and W1-LP in \cite{Petzka18}, respectively, and evaluate their performance every $100$ training epochs using a $1$-nearest neighbor ($1$-NN) classifier.\footnote{This is the same performance metric as in \cite{Seguy18}, but \cite{Seguy18} evaluate their models only once at the end of training.} 
Figure~\ref{fig: High-D USPS-MNIST 1-NN accuracy} presents the results. 
The performance of W2-FE with $K=1$ is on par with, if not better than, that of W1-LP: while their performance are very similar after epoch 4000, W2-FE with $K=1$ achieves a much higher accuracy rate early on (before epoch 4000). When we raise the persistency level to $K=3$, W2-FE converges significantly faster and consistently achieves a markedly higher accuracy rate. Indeed, it takes W2-FE with $K=1$ 8000 epochs to attain its ultimate accuracy rate, which is achieved by W2-FE with $K=3$ by epoch 3000; W2-FE with $K=3$ continues to improve after epoch 3000, yielding the best accuracy rate among all the models. Raising persistency level further to $K=5$ worsens the training quality, which might result from overfitting.


Note that the same domain adaptation experiment was carried out under various optimal-transport based methods; see \cite{Courty2017}, \cite{Seguy18}, \cite{Hamri2022}, among others. These methods directly learn the optimal transport map from the initial distribution $\mu_0$ to the target distribution $\mud$. Their goal is to change $\mu_0$ into $\mud$ {\it in one shot}, using the single optimal transport map learned. This is distinct from our gradient-flow approach that changes $\mu_0$ into $\mud$ {\it iteratively}, using numerous intermediate optimal transport maps estimated recursively. As shown in Table~\ref{tab: DA Comparison}, our algorithm (W2-FE with $K=3$) achieves the best accuracy rate 0.86, while the accuracy rate of the ``one-shot'' methods in \cite{Courty2017}, \cite{Seguy18}, and \cite{Hamri2022} ranges from 0.69 to 0.78. This suggests potential practical advantage of our iterative approach over ``one-shot'' methods and it is of interest as future research to examine this potential advantage in detail through extensive numerical experiments.

\begin{table}[h!]
    \centering
    \begin{tabular}{|c|c|}
        \hline
        Model & 1-NN classifier accuracy \\ \hline
        W2-FE (with $K=3$) & \textbf{0.8605} \\ \hline
        OT-GL \cite{Courty2017} & 0.6996 \\ \hline
        Barycentric-OT \cite{Seguy18} & 0.7792 \\ \hline
        HOT-DA \cite{Hamri2022} & 0.7639 \\ \hline    
    \end{tabular}
    \vspace{0.1in}
    \caption{1-NN accuracy for domain adaptation from MNIST to USPS dataset under different models. Values in the last three rows are quoted from \cite{Courty2017}, \cite{Seguy18}, and \cite{Hamri2022} directly.}
    \label{tab: DA Comparison}
\end{table}


\begin{figure}
    \centering
    \includegraphics[width=15cm]{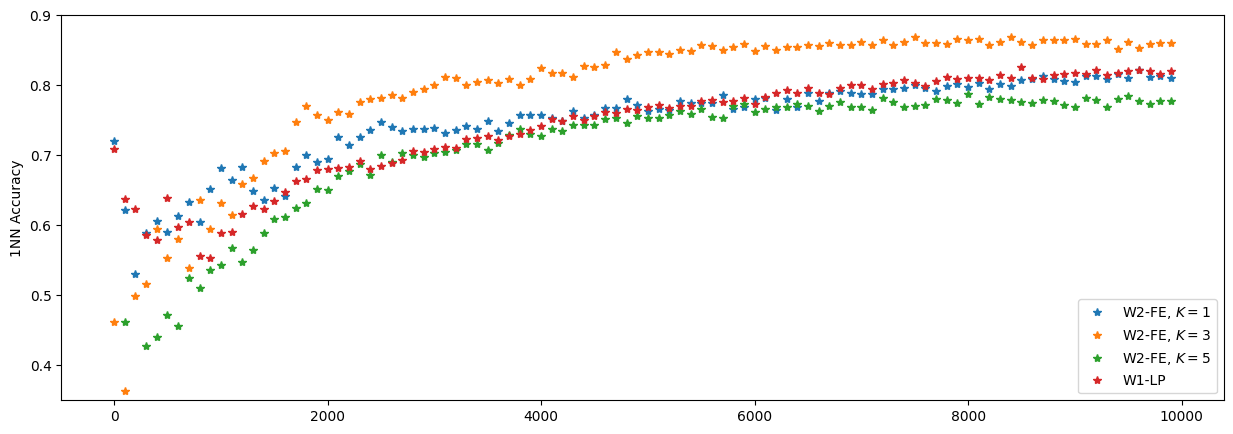}
    \caption{$1$-NN classifier accuracy against training epoch for domain adaptation from the MNIST dataset to the USPS dataset.}
    \label{fig: High-D USPS-MNIST 1-NN accuracy}
\end{figure}

\subsection{Discussion: Can We Add Persistent Training to WGAN?}
Given that persistent training can significantly improve the performance of W2-FE, it is tempting to believe that our benchmark W1-LP can be improved in the same way, i.e., by updating the generator $K\in\N$ times with the same minibatch. In fact, persistent training may {\it hurt} the performance of a WGAN algorithm, like W1-LP. To see this, recall the min-max formulation of WGAN, i.e., 
\begin{align*} 
    &\min_{\theta} W_{1}\big((G_{\theta})_{\#}\mu^{\bm{z}}, \mud\big)\notag =\min_{\theta} \max_{||\phi||_{\operatorname{Lip}}\leq1} \left\{ \mathbb{E}_{\bm{z}\sim \mu^{\bm{z}}} [\phi(G_{\theta}(\bm{z}))] - \mathbb{E}_{\bm{x}\sim\mud}[\phi(\bm{x})]  \right\},
\end{align*}
where $(G_{\theta})_{\#}\mu^{\bm z}$ is the probability measure on $\R^d$ induced by $G_\theta(\bm{z})$ with $\bm{z}\sim\mu^{\bm z}$, and the equality follows from the duality of $W_1$; see \cite[Section 2.2]{Gulrajani17} for an equivalent min-max setup. 
For any fixed $\theta$, by \cite[Particular Case 5.4 and Theorem 5.10(iii)]{Villani09}, the inside maximization over $||\phi||_{\operatorname{Lip}}\leq1$ admits a maximizer $\phi_\theta$ (i.e., an optimal discriminator in response to the generator $G_\theta$), whence we obtain
\begin{align} \label{min phi_theta}
    &\min_{\theta} W_{1}\big((G_{\theta})_{\#}\mu^{\bm z}, \mud\big)=\min_{\theta}  \left\{ \mathbb{E}_{\bm{z}\sim \mu^{\bm z}} [\phi_\theta(G_{\theta}(\bm{z}))] - \mathbb{E}_{\bm{x}\sim\mud}[\phi_\theta(\bm{x})]  \right\}.
\end{align}
Ideally, when updating $G_\theta$, one should follow the minimization \eqref{min phi_theta}, which takes into account the dependence on $\theta$ of the optimal discriminator $\phi_\theta$. The generator update in WGAN, however, assumes that $\phi$ is fixed (i.e., neglecting its dependence on $\theta$) and performs 
\begin{align} \label{min phi}
\min_{\theta}  \left\{ \mathbb{E}_{\bm{z}\sim \mu^{Z}} [\phi(G_{\theta}(\bm{z}))] -  \mathbb{E}_{\bm{x}\sim\mu_d}[\phi(\bm{x})]\right\} =\min_\theta \mathbb{E}_{\bm{z}\sim \mu^{Z}} [\phi(G_{\theta}(\bm{z}))],
\end{align}
which results in the generator update rule
\begin{equation} \label{Eq: WGAN Update}
    \theta \leftarrow \theta - \frac{\gamma_{g}}{m} \nabla_{\theta} \sum_{i=1}^{m} \phi(G_{\theta}(z_{i})). 
\end{equation}
Following \eqref{min phi}, but not \eqref{min phi_theta}, prevents WGAN from accurately minimizing $W_{1}\big((G_{\theta})_{\#}\mu^{\bm{z}}, \mu_d\big)$---after all, it is \eqref{min phi_theta} 
that equals $\min_{\theta} W_{1}\big((G_{\theta})_{\#}\mu^{\bm z}, \mud\big)$. 
This issue may be exacerbated by performing \eqref{Eq: WGAN Update} $K>1$ times with the same minibatch $\{ z_{i} \}$ (i.e., enforcing persistent training in WGAN). Indeed, in so doing, one may better approximate the value of \eqref{min phi}; but since the values of \eqref{min phi} and \eqref{min phi_theta} are in general distinct, getting closer to \eqref{min phi} may amount to moving further away from \eqref{min phi_theta}, the desired $\min_{\theta} W_{1}\big((G_{\theta})_{\#}\mu^{\bm z}, \mud\big)$. 
As a quick check, we redo the two-dimensional experiment using W1-LP with varying $K\in\N$ values. Figure~\ref{Fig: Persistent training for W1-LP} (the counterpart of Figure \ref{fig: 2D Quantitative Results} under W1-LP) shows that a larger $K$ not only fails to improve the eventual training quality (the smallest loss is achieved under $K=1$, i.e., no persistent training) but destabilizes the entire training process. 

\begin{figure}[h!]
    \centering
    \includegraphics[width=1\linewidth]{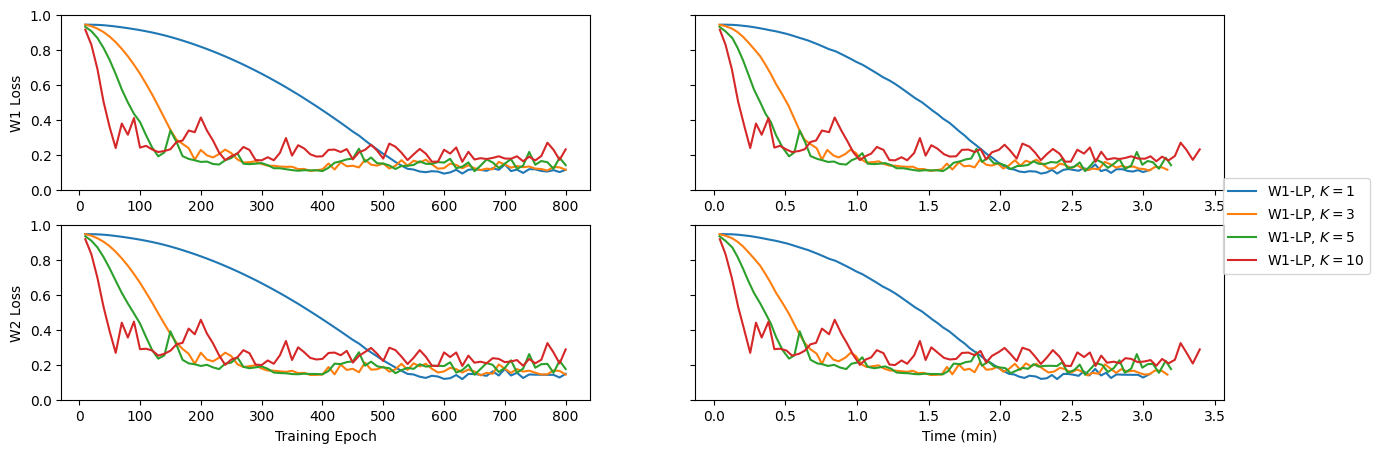}
    \caption{The first (resp.\ second) row plots the Wasserstein-$1$ (resp.\ Wasserstein-2) loss against training epoch and wall-clock time, respectively, for learning a mixture of Gaussians on a circle from an initial Gaussian at the center.}
    \label{Fig: Persistent training for W1-LP}
\end{figure}

Because of the above theoretical and numerical findings, we advise caution in the use of persistent training in any WGAN algorithm and choose to keep W1-LP as it is (without persistent training) when comparing it with W2-FE in our numerical experiments presented before.




\bibliographystyle{siam}
\bibliography{refs}

\end{document}